\documentclass[journal]{IEEEtran}
\pdfoutput=1 

\usepackage[export]{adjustbox}     
\usepackage[ruled,linesnumbered]{algorithm2e}
\usepackage{amssymb,amsmath}       
\usepackage{amsthm}                
\usepackage{array}
\usepackage{bm}                    
\usepackage{calrsfs}               
\DeclareMathAlphabet{\pazocal}{OMS}{zplm}{m}{n}
\usepackage[font=footnotesize,skip=5pt]{caption}
\usepackage{color}

\usepackage[inline]{enumitem}
\usepackage{float}
\usepackage{graphicx,psfrag}       
\usepackage{hhline}
\usepackage{hyperref}
\usepackage{multirow}              
\usepackage[caption=false,font=footnotesize]{subfig}
\usepackage{tabularx}              
\usepackage{tikz}
\usetikzlibrary{positioning}
\usetikzlibrary{decorations.pathreplacing}

\usepackage{wrapfig}

\let\oldnl\nl
\newcommand{\nonl}{\renewcommand{\nl}{\let\nl\oldnl}}

\newcounter{counter1}
\newtheorem{definition}[counter1]{Definition}
\newcounter{counter2}
\newtheorem{assumption}[counter2]{Assumption}
\newcounter{counter3}
\newtheorem{lemma}[counter3]{Lemma}
\newcounter{counter4}
\newtheorem{problem}[counter4]{Problem}
\newcounter{counter5}
\newtheorem{proposition}[counter5]{Proposition}

\newcommand{\comment}[1]{}
\newcommand{\ie}{\textit{i.e.}}
\newcommand{\mc}[1]{\mathcal{#1}}
\newcommand{\pc}[1]{\pazocal{#1}}
\newcommand{\subt}[2]{#1_{\text{#2}}}
\newcommand{\norm}[1]{\left\lVert #1 \right\rVert}

\DeclareMathOperator{\interior}{int}

\renewcommand{\tt}{\fontfamily{cmtt}\selectfont}
\newcommand{\func}[1]{{\tt{#1}}}
\newcommand{\kw}[1]{{\tt{#1}}}

\newcommand{\config}{\bm{c}}
\newcommand{\GP}{\mc{G} \cap \mc{P}}
\newcommand{\GPi}{\mc{G} \cap \mc{P}_i}

\newcommand{\TA}{{\tt TypeA} }
\newcommand{\TB}{{\tt TypeB} }

\graphicspath{{figures/}}

\makeatletter
\newcommand\np{Note\;to\;Practitioners}
\def\NP{
  \normalfont\small%
  \bfseries\textit{\np}---\relax
  \@IEEEgobbleleadPARNLSP
}

\makeatother

\title{\LARGE \bf A Certified-Complete Bimanual Manipulation Planner}
\author{Puttichai Lertkultanon and Quang-Cuong Pham}
\date{}

\begin{document}
\maketitle
\thispagestyle{empty}
\pagestyle{empty}

\setlength{\algomargin}{2em}
\SetKwIF{If}{ElseIf}{Else}{if}{then}{elif}{else}{}%
\SetKwRepeat{Do}{do}{while}%
\SetKwFor{ForEach}{for each}{do}{}%
\AlgoDontDisplayBlockMarkers\SetAlgoNoEnd\SetAlgoNoLine%
\DontPrintSemicolon
\newcommand\mycommfont[1]{\tt{#1}}
\SetCommentSty{mycommfont}

\begin{abstract}
  Planning motions for two robot arms to move an object
  collaboratively is a difficult problem, mainly because of the
  closed-chain constraint, which arises whenever two robot hands
  simultaneously grasp a single rigid object. In this paper, we
  propose a manipulation planning algorithm to bring an object from an
  initial stable placement (position and orientation of the object on
  the support surface) towards a goal stable placement. The key
  specificity of our algorithm is that it is certified-complete: for a
  given object and a given environment, we provide a certificate that
  the algorithm will find a solution to \emph{any} bimanual
  manipulation query in that environment whenever one
  exists. Moreover, the certificate is constructive: at run-time, it
  can be used to quickly find a solution to a given query. The
  algorithm is tested in software and hardware on a number of large
  pieces of furniture.

\end{abstract}

\begin{NP}
  This paper presents an algorithm to solve a difficult class of
  bimanual manipulation planning problems where a movable object can
  be moved only when grasped by two robots. These problems arise
  naturally when manipulating a large and/or heavy object such as a
  piece of furniture. With a given object and environment, we provide
  a method to compute a certificate that the algorithm will find a
  solution to \emph{any} bimanual manipulation query in that
  environment whenever one exists. The certificate can also be used to
  quickly construct a solution to a given query. The algorithm is
  tested in software and hardware on a number of large pieces of
  furniture.
\end{NP}


\section{Introduction}
\label{section:introduction}

Large or heavy objects are best manipulated using two hands. Humans
are good at bimanual manipulation: think of how we can, for example,
effortlessly manipulate a large piece of furniture
(Fig.~\ref{figure:illustration_human}). By contrast, bimanual
manipulation is still challenging for robots, mainly because of the
closed-chain constraint, which arises whenever two robot hands
simultaneously grasp a single rigid object
(Fig.~\ref{figure:illustration_robots}). This constraint poses
significant challenges for manipulation planning since it (i) reduces
the dimension of the configuration space~\cite{YLK01tra}, and (ii)
restricts the range of motion of each robot
arm~\cite{XLP16arxiv}. Thus, while unimanual manipulation planning is
a relatively established research field with solid theoretical
foundations and a number of working demonstrations~(see e.g.,
\cite{SimX04ijrr, GMS15tase, WanX15icra, LP15ral} and references
therein), results in bimanual manipulation planning are still scarce,
see Section~\ref{section:related_works} for a review.

This paper specifically considers the harder class of problem
instances where the manipulated object can be moved \emph{only} when
grasped with both hands, which is the case for large or heavy
objects. We propose a manipulation algorithm to bring the object from
an initial stable placement (position and orientation of the object on
the support surface) towards a goal stable placement. The algorithm
works at two levels: task-planning and motion-planning. At the
task-planning level, a sequence of stable intermediate placements
where the object can be ungrasped and regrasped is found. At the
motion-planning level, the motions of the two arms (when they carry
the object between two intermediate placements or when they move
freely while the object is at an intermediate placement) are
determined.

The key specificity of our algorithm is that it is
\emph{certified-complete}: for a given object and a given environment,
we provide a certificate that the algorithm will find a solution to
\emph{any} bimanual manipulation query in that environment whenever
one exists. Moreover, the certificate is constructive: at run-time, it
can be used to quickly find a solution to a given query. The algorithm
is tested in software and hardware on a number of large pieces of
furniture. An implementation is openly available at
{\tt{https://gitlab.com/puttichai/pymanip}}.

Proofs of completeness have been obtained for some classes of motion
planning algorithms, under more or less restrictive and verifiable
assumptions~\cite{Lav06book, CPN17ras}. However, to our knowledge,
there currently exists no complete or certified-complete bimanual
manipulation planner. This is because, in addition to the
motion-planning level, manipulation planners include the task-planning
level, whose completeness properties are difficult to formalize and to
prove. Nevertheless, completeness results are crucial for automation,
where time is a valuable asset. Certified-completeness, for example,
eliminates the need to spend computation time searching for
non-existent manipulation paths. Computed certificates also help the
algorithm to find shorter manipulation paths since the robots will
only bring the object to different placements only if necessary.

\begin{figure}[tp]
  \centering
  \subfloat[{}]{
    \includegraphics[height=0.22\textwidth]{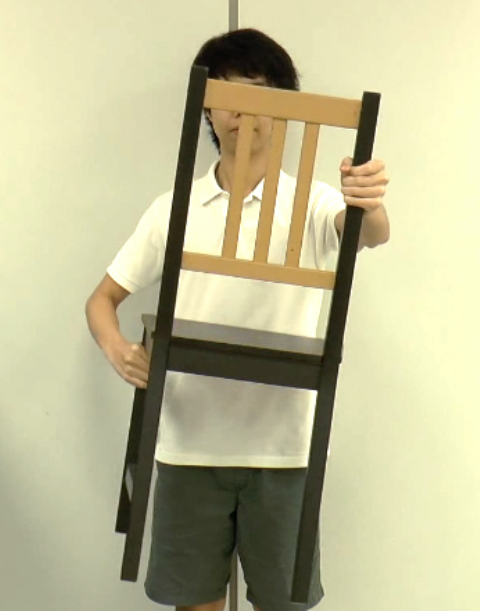}
    \label{figure:illustration_human}
  }
  \subfloat[{}]{
    \includegraphics[height=0.22\textwidth]{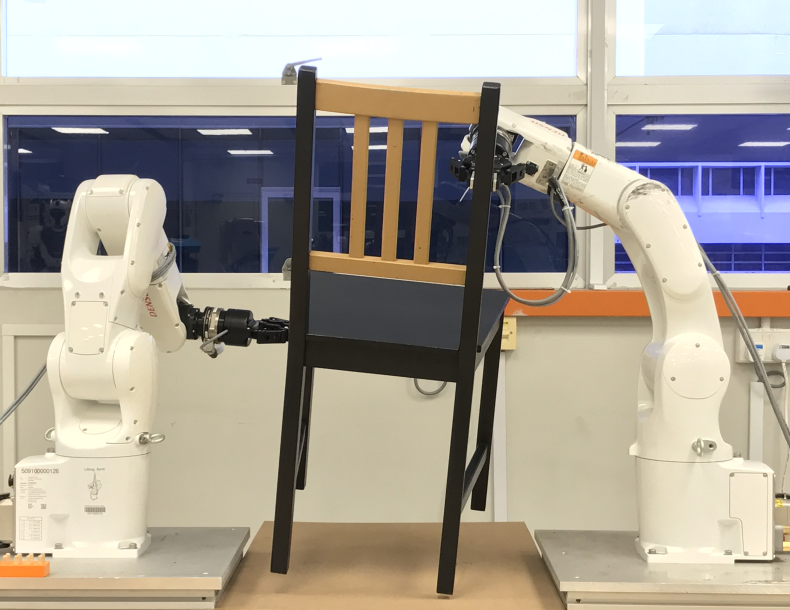}
    \label{figure:illustration_robots}
  }
  \caption{(a) A human can effortlessly manipulate a large piece of
    furniture. (b) Unlike humans, bimanual manipulation is extremely
    challenging for robots due to, for example, closed-chain
    constraints.}
  \label{fig:illustration}
\end{figure}

The rest of the paper is organized as follows. In
Section~\ref{section:related_works}, we review related works in
manipulation planning. In Section~\ref{section:overview}, we introduce
the background of manipulation planning and give an overview of the
proposed bimanual manipulation planner. In
Section~\ref{section:internal} and Section~\ref{section:external}, we
discuss main components of the proposed planner and introduce the
notion of certificate. In Section~\ref{section:experiments}, we
present software and hardware experiments to validate our
approach. Finally, in Section~\ref{section:conclusion}, we discuss the
advantages and limitations of the approach and sketch some direction
for future work.

\section{Related Works}
\label{section:related_works}

\subsection{Bimanual Manipulation Planning}
In a pioneering work~\cite{KL92icra}, the authors considered the
problems with the constraint that the movable object could only be
moved when grasped by both arms, similar to ours. The proposed
solutions were based on discretization of the configuration
space. Their applicability were therefore limited to low-dimensional
problems.

Apart from~\cite{KL92icra}, most existing works on bimanual
manipulation fall into one of the following paradigms:
\begin{enumerate}[label=\alph*), leftmargin=*, itemindent=1.6em]
\item Passing an object from one hand to another~\cite{SauX10iros,
    BC12icra, HarX14icra, WH16ral}. These motions can be seen as a way
  to increase the workspace volume of the system or to help solve
  single-arm manipulation problems more easily.

\item Focusing on control of interactions (robot-robot, robot-object,
  or object-object). Examples include assembly
  manipulation~\cite{HLC07ismc, SurX10humanoids, SN11aa,
    EK08springer} and objects handling~\cite{KWR15tase}.
  
\item Discussing only reaching motions~\cite{GCS09iros, VahX09iros,
    VahX10icra} or non-cooperative tasks~\cite{AS95mmt,
    ZacX10iros}. This case is most similar to the usual multi-arm path
  planning where robot arms move \emph{independently but
    coordinatively} to reach their goals without colliding with one
  another.\label{item:reaching}
  
\item No regrasping. In~\cite{GCS08aim, CCl13ijrr}, the start and goal
  configurations are closed-chain configurations. This case is more
  related to closed-chain motion planning.\label{item:closedchain}
  
\item Others. For example, in~\cite{LPK15iros}, the authors solved a
  manipulation problem by sequentially generating a sequence of object
  contact states, object poses, and manipulator contact points. They,
  however, did not take into account robot
  kinematics. In~\cite{PS15icra}, the authors discussed a problem of a
  bimanual robot manipulating a foldable chair. The task was solved
  via chair state discretization.
\end{enumerate}

A more thorough survey of work on bimanual manipulation can be found
in~\cite{SmiX12ras}.

The problem we are interested in, however, is a combination of various
subproblems, including~\ref{item:reaching} and \ref{item:closedchain}
mentioned above as well as the regrasping problem.

Although the problem itself is similar in nature to the one tackled
in~\cite{KL92icra}, the setting here is more practical. Our approach
can deal with problems with high degrees-of-freedom
(DOF). Furthermore, while the planners presented in~\cite{KL92icra}
did not have any performance guarantee, our proposed planner is
certified-complete.

\subsection{Completeness Results in Manipulation Planning}
Manipulation planning can be viewed as a member of a broader class of
problems, so-called \emph{multi-modal motion planning}, in which the
configuration space has multi-modal structure and each mode limits
possible motions to a submanifold. For the manipulation planning
problem that we consider here, the available \emph{family of modes} of
motions consist of \emph{transit} and \emph{transfer}; each family
comprises infinitely many modes. For example, each mode in the transit
family corresponds to one object placement.

In~\cite{HL10ijrr}, the authors presented a probabilistically complete
multi-modal planner, which, however, applicable to only problems with
\emph{finite} modes. A more relevant result was published
in~\cite{HN11ijrr} where the authors presented a planner, Random-MMP,
which can deal with infinite number of modes. Despite the
probabilistic completeness guarantee, using Random-MMP directly to
solve a pick-and-place task poses critical disadvantages. This is
mainly due to unsupervised mode switches: consider for example when
the object is at rest on a supporting surface without being grasped by
any robot (i.e., the configuration is in a transit mode), Random-MMP
will proceed by simply sampling an adjacent mode, which is essentially
sampling \emph{any} grasp (transfer mode). Without utilizing knowledge
of how grasps and placements correlates, as is done, e.g.,
in~\cite{LP15ral}, the planner can be very slow since it indeed needs
to randomly sample a correct order of a correct combinations of grasps
and placements before it can eventually reach the goal. Furthermore,
the probabilistic completeness of Random-MMP relies on the
expansiveness of the space of all modes, which is difficult, if not
impossible, to characterize or verify.

In this work, we introduce a more practical notion of completeness,
namely certified-completeness. A certified-complete planner computes a
certificate which guarantees existence of solutions to any feasible
manipulation query. Although the computation of certificates itself is
not complete, i.e., there currently exists no theoretical guarantee if
such computation will be successful, we show that it is in fact
practical to compute such certificates, as presented in
Section~\ref{section:experiments} for a number of realistic cases.

Note also that there is also another somewhat related line of
research, in which the focus is on computation of space disconnection
certificate (see, e.g.,~\cite{BreX05afr}).

\subsection{Regrasping}
Generally speaking, regrasping is a grasp-changing operations. Here we
are interested in the case when manipulators are equipped with
parallel jaw grippers, which are the most common and robust grippers
in the industry. Unlike multi-fingered hands which can perform in-hand
regrasping, a robot equipped with a parallel gripper has to rely on a
\emph{support surface}, on which the object can be placed stably while
ungrasped, to change the grasp.

Works on regrasping utilizes the knowledge that to realize any
regrasping motions, the robot(s) must place the object down on the
support surface. The system configuration has naturally to satisfy two
criteria:
\begin{enumerate*}
\item the robot(s) must be grasping the object and
\item the object must be at a stable position.
\end{enumerate*}
The set of configurations satisfying the aforementioned criteria,
denoted as $\GP$, and connectivity between its different connected
component play significant roles in solving regrasping problems.

Pioneering works on regrasping problems, including~\cite{TP87icra,
  RW97icra, StoX99icra}, characterized the set $\GP$ by means of
discretization. Their methods are therefore limited in a number of
ways. However, the authors of~\cite{TP87icra} also proposed an
interesting notion of Grasp-Placement Table, based on the
discretization of $\GP$, which captured the connectivity of
$\GP$. More recent work on regrasping such as~\cite{WanX15icra,
  LP15ral} also employed some kinds of graphs to represent the
connectivity.

The set $\GP$ can, in fact, be grouped into a finite number of
subsets, called grasp classes and placement
classes~\cite{LP15ral}. Utilizing these facts, the authors
of~\cite{LP15ral} introduced a high-level Grasp-Placement Graph which
showed potential connectivity between different connected components
of $\GP$. They proposed a manipulation planner which, with the
guidance from the graph, explored the configuration space efficiently
and systematically.

One possible way to solve a bimanual manipulation planning problem is
then to extend the high-level Grasp-Placement Graph, originally
proposed for unimanual systems, to bimanual cases. However, the
combinatorial complexity grows much too high, making this approach not
suitable even in the case when the object has a moderate number of
grasp classes. For example, consider a unimanual setting. Suppose the
start and goal placements have $m$ grasps in common but no transfer
path directly connecting the two placement classes exists. The planner
will have to explore \emph{exhaustively} all $m$ paths connecting
placements start and goal in the graph before considering any
manipulation path with some intermediate placements. In a bimanual
setting, the planner will have to explore all $\pc{O}(m^2)$
possibilities in case no direct transfer path exists between the start
and goal placements. Suppose there are tens of common grasp classes
between the start and goal placements, this means that the planner
needs to explore already hundreds of possibilities before trying to
plan a manipulation paths with one intermediate placements.

\subsection{Motion Planning with Closed Kinematic Chains}
Closed-chain motion planning is by itself a difficult and challenging
problem. Efficient path planners such as Rapidly-exploring Random
Trees (RRTs)~\cite{LK01acr} or their variants cannot be directly
applied to solve such problems since the probability of a randomly
sampled configuration satisfying the closed-chain constraint is
essentially null~\cite{YLK01tra}. This is because the set of valid
closed-chain configurations forms a set of manifolds of dimension
lower than that of the ambient space. To cope with this issue, various
methods have been devised to sample closed-chain configurations and to
interpolate closed-chain trajectories.

Random gradient descent was used in~\cite{YLK01tra} to move a randomly
sampled configuration towards the constraint
manifold. In~\cite{HA00wafr}, the authors proposed to break the closed
kinematic chain into two subchains. A configuration of one subchain is
sampled randomly while a configuration of the other is computed so as
to close the chain. This method was further improved
in~\cite{CSL02icra}. More recent work samples configurations on a
tangent space of the constraint manifold~\cite{JP13tro,
  KimX16robotica}.

We take a different approach to closed-chain motion
planning. Essentially, to interpolate a trajectory between
closed-chain configurations, our planner first interpolates a
trajectory for the movable object. The trajectory is then tracked by
the two robots. We describe our closed-chain motion planner as well as
our rationale in Section~\ref{section:external}.

\section{Background and Overview of the Bimanual Manipulation Planning
  Algorithm}
\label{section:overview}

\subsection{Background}
This Section presents definitions and fundamentals of bimanual
manipulation planning built based on previous works~\cite{SimX04ijrr,
  LP15ral}.

Consider the $3$D space where the bimanual manipulation system is
located, called world. The world, $\pc{W}$, consists of two robots
$\pc{R}_1$ and $\pc{R}_2$, a movable object $\pc{O}$, and the
environment $\pc{E}$. Each robot is equipped with a parallel jaw
gripper. The environment also includes \emph{support surface(s)} on
which the object is allowed to rest.

Let $\mc{C}_{\pc{R}_1}$ and $\mc{C}_{\pc{R}_2}$ be the configuration
spaces of the two robots and $\mc{C}_{\pc{O}} \subseteq SE(3)$ the
configuration space of the object. The composite configuration
$\mc{C}$ is defined as the Cartesian products of the three
aforementioned spaces. Each composite configuration
$\config \in \mc{C}$ can then be written as
$\config = (\bm{q}_1, \bm{q}_2, \bm{T})$, where
$\bm{q}_1 \in \mc{C}_{\pc{R}_1}$, $\bm{q}_2 \in \mc{C}_{\pc{R}_2}$ and
$\bm{T} \in \mc{C}_{\pc{O}}$.

We equip with the composite configuration space a metric $d$ defined
as a linear combination of Euclidean distance between robot
configurations and a distance between the object transformation
matrices. In particular,
$d(\bm{c}_a, \bm{c}_b) = \alpha\left(\norm{\bm{q}_{1a} -
    \bm{q}_{2a}}_2 + \norm{\bm{q}_{1b} - \bm{q}_{2b}}_2 \right) + (1 -
\alpha)w(\bm{T}_a, \bm{T}_b)$, where $\alpha \in [0, 1]$ and $w$ is
the weighted sum of the minimal geodesic distance between two
rotations~\cite{PR97tg} and the Euclidean distance between two
displacements.

We now define a \emph{grasp} and a \emph{placement} as follows.

\begin{definition}
  A \textbf{grasp} is a relation between the object pose and the
  grippers' poses.
\end{definition}

One can represent a grasp by, e.g., a pair of relative transformations
between each robot gripper and the object. Note that from the
definition, any pair of relative transformations can be a
grasp. However, the object can be moved only when being grasped by a
\emph{valid grasp}. The set of all valid grasps are to be determined
by the users, either explicitly (e.g., as a set of grasps) or
implicitly (e.g., as conditions to be satisfied by the grippers). Note
also that there can be many pair of robot configurations
$(\bm{q}_1, \bm{q}_2)$ corresponding to exactly the same grasp due to
multiplicity of inverse kinematic (IK) solutions associated with the
same grippers' poses.

The set of all valid grasps can be \emph{parameterized} by a set of
parameters~\cite{LP15ral}, which is finite but not necessarily
unique. Consider for example an object composed entirely of
boxes\footnote{This is not a particularly impractical setting since a
  piece of furniture tends to be geometric and thus can be
  approximately, if not exactly, represented by boxes.} and a gripper
shown in Fig.~\ref{figure:gripper}. Grasp parameters may be defined as
follows~\cite{LP15ral}. $l$ is an integer indicating the index of the
link (box) that the gripper is grasping. $a$ is an integer indicating
how the gripper is approaching the object. Assuming, without loss of
generality, that each box is aligned with its local coordinate
frame. The integer $a$ may be a number from $1$ to $6$, where if
$a = 1$, the gripper's approaching direction is aligned with the
$+x$-axis of the box's local frame; if $a = 2$, the gripper's
approaching direction is aligned with the $+y$-axis, etc. $b$ is an
integer indicating which axis of the box's local frame the gripper's
sliding direction is aligned with. And the last
parameter\footnote{Note that we can have more grasp parameters. For
  example, the gripper could tilt around the (virtual) axis connecting
  the two finger tips. We may introduce another parameter $\theta$ to
  indicate the tilting angle.} $\delta$ is a real number indicating
the position of the gripper along the sliding direction. For example,
if the gripper is grasping the box at the middle, we may assign
$\delta = 0$, and $\delta$ increases (or decreases) when the gripper
\emph{slides} along the sliding direction. Using this notion, a grasp
for the $i^{\text{th}}$ robot may be written as a vector
$g_i = [l_i\, a_i\, b_i\, \delta_i]^\top$ and therefore a bimanual
grasp may be written as $g = \left[g_1^\top g_2^\top\right]^\top$.

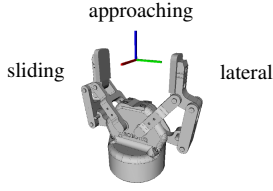
\begin{figure}
  \centering
  \begin{tikzpicture}
    \node at (0, 0) {\includegraphics[width = 0.1\textwidth] {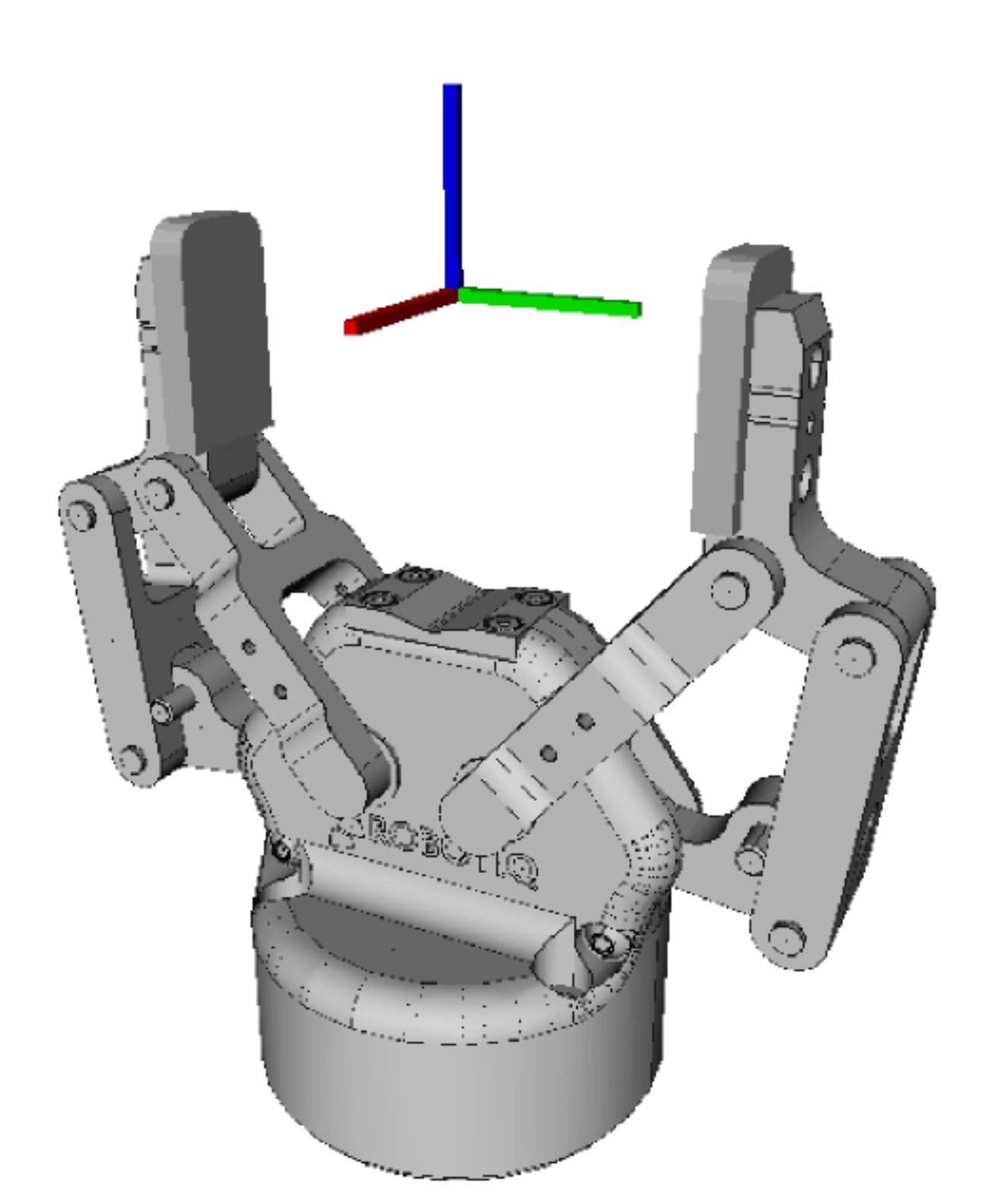}};
    \node [scale=0.8] at ( 1.4, 0.4) {lateral};
    \node [scale=0.8] at (-1.4, 0.4) {sliding};
    \node [scale=0.8] at ( 0, 1.2) {approaching};
  \end{tikzpicture}
  \caption{A parallel gripper with its local frame. The lateral
    direction is orthogonal to both finger surfaces. The sliding
    direction is parallel to both finger surfaces and is defined such
    that the approaching direction is pointing out of the gripper.}
  \label{figure:gripper}
\end{figure}

\begin{definition}
  A \textbf{placement} refers to an object transformation at which the
  object is in contact with a support surface.

  A placement is said to be \textbf{stable} if when not in contact
  with any robot, the object remains stationary.
\end{definition}

The set of all stable placements can be seen as a set of
$SE(2)$. They therefore can be parameterized by three parameters $x$,
$y$, and $\theta$, where $x$ and $y$ represent the position of some
nominal point of the object with respect to the support surface, and
$\theta$ represents the rotation around an axis passing through the
point $(x, y)$ and perpendicular to the surface.

With the above definitions of grasps and placements, we can now define
a \emph{collision-free} configuration and a \emph{feasible}
configuration.

\begin{definition}
  A composite configuration is said to be \textbf{collision-free} if
  there is no collision in the world except the ones induced by valid
  grasps and ones induced by placements.
\end{definition}

\begin{definition}
  A composite configuration $\bm{c} = (\bm{q}_1, \bm{q}_2, \bm{T})$ is
  said to be \textbf{non-singular} if the Jacobians of the two robots,
  $\bm{J}_{\pc{R}_1}(\bm{q}_1)$ and $\bm{J}_{\pc{R}_2}(\bm{q}_2)$,
  have maximal rank. Otherwise, the configuration is said to be
  \textbf{singular}.
\end{definition}

\begin{definition}
  A composite configuration is said to be \textbf{feasible} if it is
  collision-free and non-singular and at least one of the following
  holds:
  \begin{enumerate*}
  \item The robots are grasping the object with a valid grasp;
  \item The object is at a stable placement.
  \end{enumerate*}
\end{definition}

Let us now consider intrinsic structure of $\mc{C}$. For convenience,
we define a function $\pi_p: \mc{C} \rightarrow \mc{C}_{\pc{O}}$ which
projects a composite configuration
$\config = (\bm{q}_1, \bm{q}_2, \bm{T})$ into $SE(3)$ such that
$\pi_p(\config) = \bm{T}$. There are two types of subsets of $\mc{C}$
induced by valid grasps and stable placements.

\begin{definition}
  Grasp configuration set, $\mc{G}$, is the set of feasible composite
  configurations where the robots are grasping the object with a valid
  grasp.
\end{definition}

\begin{definition}
  Placement configuration set, $\mc{P}$, is the set of feasible
  composite configurations such that
  \begin{enumerate}[]
  \item $\forall \bm{c} \in \mc{P}$
    $\pi_p(\bm{c}) \text{ is a stable placement}$ \emph{and}
  \item $\forall \bm{c} \in \mc{P}$ $\exists \bm{c}' \in \mc{G}$
    $ \pi_p(\bm{c}') = \pi_p(\bm{c}) $.
  \end{enumerate}
\end{definition}

The second requirement of the placement configuration set is to ensure
that for any placement configuration $\bm{c} \in \mc{P}$, its
corresponding placement is always reachable by some grasp.

Both $\mc{G}$ and $\mc{P}$ can be partitioned into a \emph{finite}
number of \emph{grasp classes} and \emph{placement classes},
respectively~\cite{LP15ral}. From the grasp parameters we introduced
earlier, we define a grasp class as a subset of $\mc{G}$ whose
configurations have the same grasp parameters $l$ (link index) and $a$
(approaching direction). For example, if the object is a box, there
will be $6$ grasp classes in total. Now consider partitioning of
$\mc{P}$. Let $\pc{H}$ be the convex hull of the object. All stable
placements can be grouped based on which surface of $\pc{H}$ is in
contact with the support surface. Therefore, a placement class is
defined as a subset of $\mc{P}$ where at each configuration, the same
face of $\pc{H}$ is in contact with the support surface. For
convenience, we will also say that two object transformations are
\emph{in the same placement class} if at both transformations, the
same face of the convex hull $\pc{H}$ is in contact with the support
surface.

There are two types of physically realizable \emph{single-mode} paths:
transit and transfer. A transit path is a path in $\mc{P}$ where the
placement remains unchanged throughout while a transfer path is a path
in $\mc{G}$ where the grasp remains unchanged throughout. A
manipulation path is defined as an alternating sequence of single-mode
paths. To plan a manipulation path, a \emph{manipulation query} must
be provided to a planner. A manipulation query is defined as follows.

\begin{definition}
  A \emph{manipulation query}, or simply \emph{query}, $Q$, is a set
  of information provided to a manipulation planner to solve for a
  manipulation trajectory. A query consists of at least a pair of
  stable placements, $\bm{T}_s$ and $\bm{T}_g$, which are the start
  and goal object transformations.

  A query is said to be feasible if
  $\bm{T}_s, \bm{T}_g \in \pi_p(\mc{P})$.
\end{definition}

Then a manipulation planning problem can be stated as follows.

\begin{problem}
  Given the description of the world and a query
  $Q = (\bm{T}_s, \bm{T}_g)$, find a manipulation trajectory which
  brings the object from $\bm{T}_s$ to $\bm{T}_g$.
\end{problem}

\subsection{Overview of the Proposed Bimanual Manipulation Planning
  Algorithm}

We propose the following approach to solving a bimanual manipulation
query:
\begin{enumerate}[label={\textbf{Step \arabic*}},leftmargin=3.5em]
\item Identify the placement classes of $\bm{T}_s$ and $\bm{T}_g$\ as
  $\mc{P}_s$ and $\mc{P}_g$,
  respectively.\label{algorithm:concept-initialization}
\item Generate \TA trajectory, within the placement class $\mc{P}_s$,
  to move the object from $\bm{T}_s$ to some
  $\bm{T}'_s$.\label{algorithm:concept-TA}
\item Generate \TB trajectory to bring the object from $\bm{T}'_s$ to
  some $\bm{T}'_g$ in the placement class
  $\mc{P}_g$.\label{algorithm:concept-TB}
\item Generate \TA trajectory, within the placement class $\mc{P}_g$,
  to move the object from $\bm{T}'_g$ to
  $\bm{T}_g$.\label{algorithm:concept-finalTA}
\end{enumerate}

A solution to a query will be a sequence of \TA trajectories, which
connect configurations in the same placement class, and \TB
trajectories, which connect configurations from different placement
classes.

In the above steps, $\bm{T}'_s$ (respectively $\bm{T}'_g$) is an
object transformation which can serve as an initial (respectively
goal) transformation of the to-be-generated \TB trajectory. Note also
that in some cases, one may need to generate \TB trajectories to move
the object to, and between, some \emph{intermediate} placements since
a direct connection between $\mc{P}_s$ and $\mc{P}_g$ may not exist or
cannot be found. The procedure can be done by repeating Step 2 and
Step 3 until the goal placement $\mc{P}_g$ is reached.

The completeness of the above approach depends on the completeness of
\TA and \TB trajectory generation methods. In the remaining of this
Section, we give brief overviews of generation of both \TA and \TB
trajectories, as well as the main algorithm.

\subsubsection{{\upshape \TA}}
To plan \TA trajectories, we argue that we can consider the set
$\mc{T}_i \subset SE(2)$ of object configurations (see
Section~\ref{section:internal} for more details) instead of examining
a placement class $\mc{P}_i$, which is a subset of the
high-dimensional $\mc{C}$.

Given two object configurations $\bm{T}_1$ and $\bm{T}_2$ in the same
connected component of $\mc{T}_i$, we first generate an object path
$\sigma$, as if it could move freely by itself on a support
surface. Then we present a procedure to generate a \TA trajectory
which moves the object along $\sigma$. We prove that, given a valid
object path, such a \TA trajectory always exists and that our
procedure will terminate with a solution in finite time.

\subsubsection{{\upshape \TB}}
Since the motions of the system are severely constrained by closed
kinematic chains, randomly generating closed-chain queries, where the
start and goal configurations are in different placement classes, has
slim chances of the queries being solvable. To resolve this issue, we
propose a heuristic to generate closed-chain queries in such a way
that, by our intuition, does not require a large range of robot
motions to solve them.

Now suppose that one has a \TB trajectory
$M^B:[0, 1] \rightarrow \mc{C}$ connecting two placement classes
$\mc{P}_i$ and $\mc{P}_j$, i.e., $M^B(0) \in \mc{P}_i$ and
$M^B(1) \in \mc{P}_j$. Observe that provided that the world does not
change, whenever one needs to connect configurations
$\bm{c}_1 \in \mc{P}_i$ and $\bm{c}_2 \in \mc{P}_j$, one can
\emph{reuse} the trajectory $M^B$ by planning two \TA trajectories,
$M^A_1$ and $M^A_2$, where $M^A_1$ connects $\bm{c}_1$ and $M^B(0)$
and $M^A_2$ connects $M^B(1)$ and $\bm{c}_2$. The composition (as
defined in~\cite{LP15ral}) of the three trajectories, i.e.,
$M = M^A_1 \ast M^B \ast M^A_2$, then serves as a solution. Since \TB
trajectories can be reused as discuss above, they have to be computed
only once and the procedure may as well be offline. This inspires us
to introduce a notion of a \emph{certificate} which is a set of useful
\TB trajectories. Once computed, a solution (if any) to any given
bimanual manipulation query can then be constructed from the
certificate in the aforementioned manner.

\subsubsection{Main Algorithm}
First, we generate a certificate $\mc{M}$. This step needs to be done
only once per problem setting. Given a query
$Q = (\bm{T}_s, \bm{T}_g)$, we then extract a placement sequence
$\mc{P}_1 \rightarrow \mc{P}_2 \rightarrow \cdots \rightarrow
\mc{P}_n$, where $\mc{P}_1 = \mc{P}_s$ and $\mc{P}_n = \mc{P}_g$,
along with their corresponding transfer trajectories
$M^B_1, M^B_2, \ldots, M^B_{n-1}$, where $M^B_i$ is a \TB trajectory
connecting the $i^\text{th}$ placement in the sequence to the
next. Next, we generate a \TA trajectory connecting $M^B_i$ and
$M^B_{i + 1}$ for every $i$. Finally, a solution to the query is
constructed by concatenating all the trajectory (using the composition
operation). Fig.~\ref{figure:main-algorithm} illustrates the
proposed algorithm.

\begin{figure}
  \centering
  \subfloat[{}]{
    \includegraphics[width=0.4\textwidth]{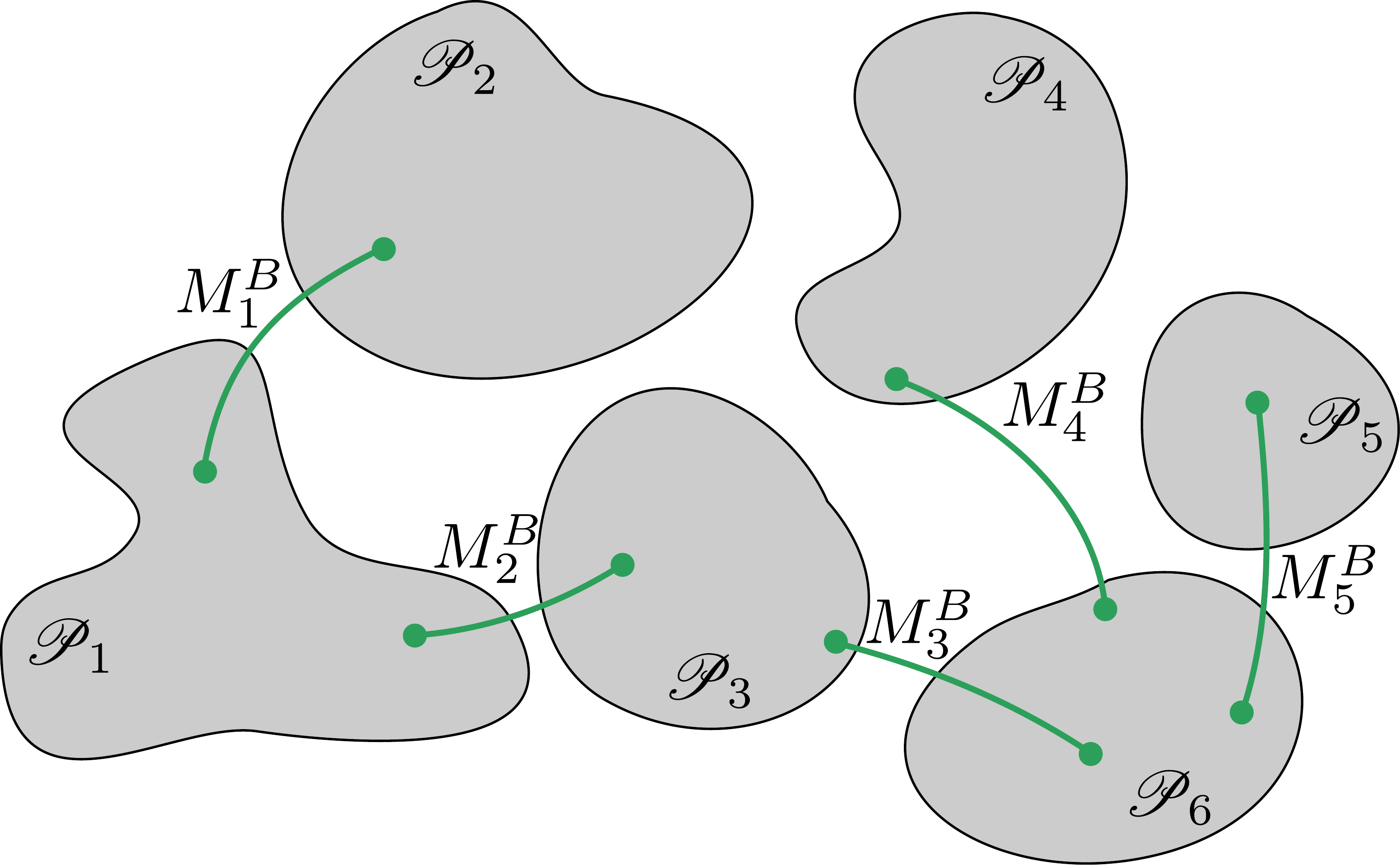}
    \label{figure:certificate}
  }
  \\
  \subfloat[{}]{
    \includegraphics[width=0.4\textwidth]{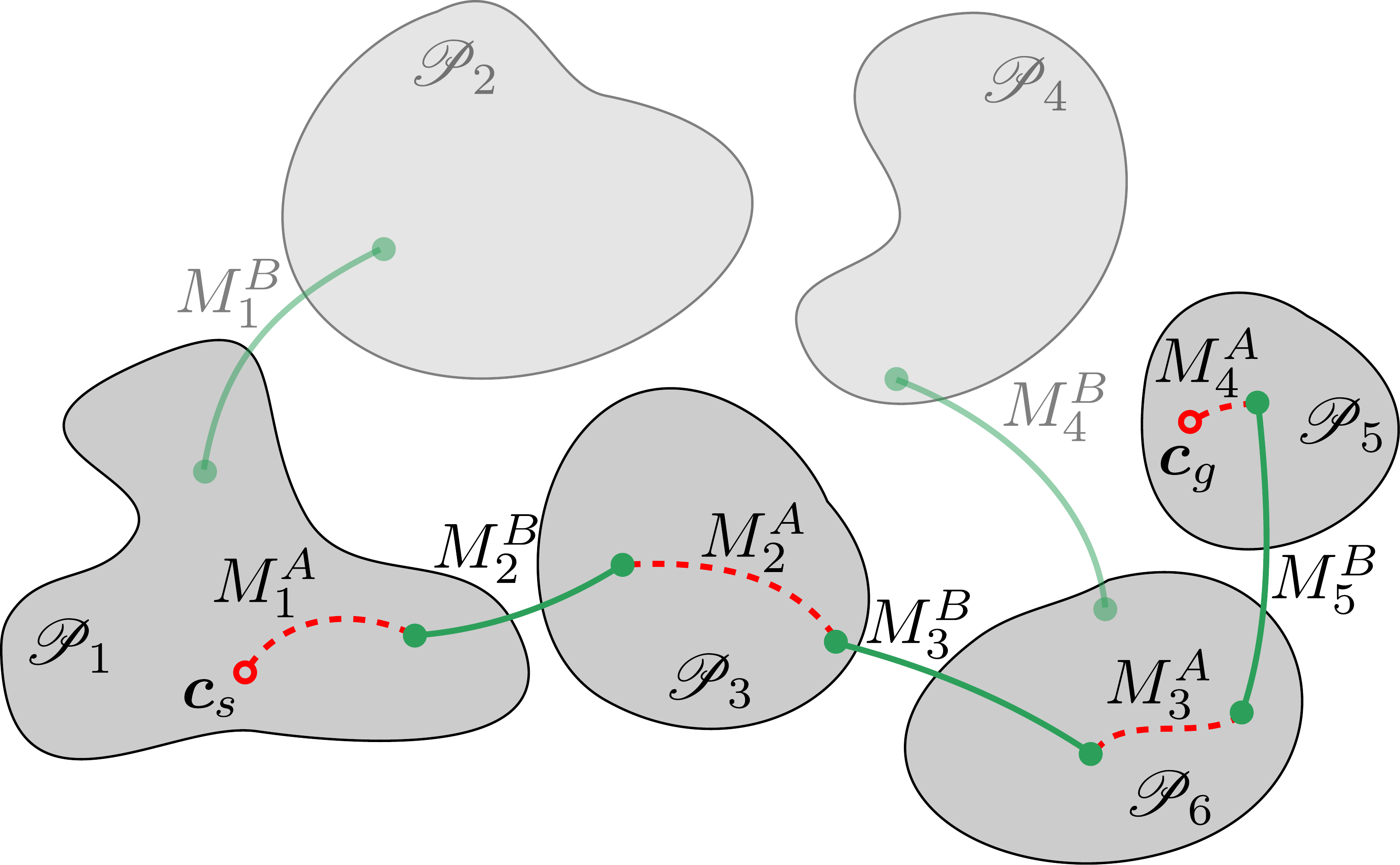}
    \label{figure:solution}
  }
  \caption{(a) Placement classes and their verified ({\tt TypeB})
    connectivities. Each $M^B_i, i \in \{1, 2, \ldots, 5\}$ is a
    transfer ({\tt TypeB}) trajectory connecting two different
    placement classes. We call the set
    $\{M^B_1, M^B_2, \ldots, M^B_5\}$ a \emph{certificate}. (b) To
    construct a solution to a query which starts in $\mc{P}_1$ and
    ends in $\mc{P}_5$, one simply plans a set of \TA trajectories,
    $M^A_1$, $M^A_2$, $M^A_3$, and $M^A_4$, as shown in dashed lines,
    to bridge the \TB trajectories. A solution trajectory is then
    $M = M^A_1 \ast M^B_2 \ast M^A_2 \ast M^B_3 \ast M^A_3 \ast M^B_5
    \ast M^A_4$.}
  \label{figure:main-algorithm}
\end{figure}

\section{Generating Trajectories within \\a Placement Class}
\label{section:internal}

In this Section, we investigate the existence of manipulation paths
connecting two composite configurations in the same placement class
({\tt TypeA}). The goal of this Section is to assert that given a path
$\sigma: [0, 1] \rightarrow SE(2)$ of the object moving from one
placement to another in the same placement class, there exists a
finite-length\footnote{We define the length of a manipulation path as
  in~\cite{LP15ral}. Generally speaking, the length is proportional to
  the number of \emph{necessary} regrasping operations along the
  manipulation path.} manipulation path associated with $\sigma$. In
other words, the projection via $\pi_p$ of the manipulation path is
$\sigma$. As the proof of existence itself is nonconstructive, we
further propose an algorithm which, given an object path $\sigma$
together with a certain set of assumptions, will return a manipulation
path associated with $\sigma$ in finite time.

\subsection{Existence of {\upshape \TA}Paths}
First we introduce the notion of \emph{single-transfer connectedness}
as follows.
\begin{definition}
  Two composite configurations $\bm{c}_1$ and $\bm{c}_2$ are
  \emph{single-transfer connected} if there exists a transfer path
  whose terminal configurations are $\bm{c}_1$ and $\bm{c}_2$.
\end{definition}

\begin{definition}
  A \emph{single-transfer connected set} $\pc{F}$ is a set in which
  any two composite configurations are single-transfer connected. If
  such a set is maximal in the sense that for any point
  $\bm{c} \in \partial \pc{F}$, every neighborhood of $\bm{c}$
  consists of both configurations that are single-transfer connected
  and not single-transfer connected with $\bm{c}$, we call it a
  \emph{single-transfer connected component}.
\end{definition}

Let $\mc{F}$ be a collection of all single-transfer connected
components in $\mc{G} \cap \mc{P}_i$. Since $\GPi$ contains no
singular configuration, any $\bm{c} \in \GPi$ must be in some
single-transfer connected component. That is,
$\GPi = \bigcup_{\pc{F} \in \mc{F}} \pc{F}$. Define $\mc{T}_i$ by
$\mc{T}_i = \pi_p\left( \interior(\GPi) \right)$. We have the
following proposition.

\begin{proposition}
  Let $\sigma: [0, 1] \rightarrow SE(2)$ be a path lying in a
  connected component of $\mc{T}_i$. Then $\sigma$ is a projection via
  $\pi_p$ of some finite-length manipulation path.
  \label{proposition:existence}
\end{proposition}

\begin{proof}
  First note that
  $\mc{T}_i = \bigcup_{\pc{F} \in \mc{F}} \pi_p\left(
    \interior(\pc{F}) \right)$. Let each projection
  $\pi_p\left( \interior(\pc{F}) \right)$ be denoted by $\pc{E}$ and
  $\mc{E}$ the collection of such sets. Since $\pi_p$ is an open map,
  each $\pc{E}$ is open. Therefore, $\mc{E}$ is an open covering of
  $\mc{T}_i$.

  Since $\sigma$ lies entirely in some connected component of
  $\mc{T}_i$, there exists a subcollection $\mc{E}'$ of $\mc{E}$ which
  covers $\sigma$. Let $\mc{I}$ be the collection of open intervals
  where each interval, $\pc{I}$, corresponds to a domain of the path
  $\sigma$ such that the path segment $\sigma(\pc{I})$ lies entirely
  in some open set $\pc{E}$.

  Now we have that $\mc{I}$ is an open covering of $[0, 1]$. Since
  $[0, 1]$ is compact~\cite{Mun00book}, there exists a finite
  subcollection of $\mc{I}$ which also covers $[0, 1]$. This means
  that the path $\sigma$ consists of a finite number of segments where
  each segment lies entirely in an open set $\pc{E}$ and hence is a
  projection of a transfer path. Therefore, we can conclude that the
  path $\sigma$ is a projection of a finite-length manipulation path.
\end{proof}

However, since the proof of compactness of $[0, 1]$ is not
constructive~\cite{CN96type}, the above proposition does not give us a
way to construct a finite-length manipulation path associated to a
given object path $\sigma$. Note also that since the proof of
Reduction Property\footnote{Reduction Property states that two
  configurations in the same connected component of
  $\mc{G} \cap \mc{P}$ are connected by some manipulation path.} given
in~\cite{ALS94wafr} also relied on the Heine-Borel covering theorem,
it also does not provide a practical way to construct a manipulation
path.

To explicitly construct an algorithm which, given an object path
$\sigma$, computes in finite time an associated finite-length
manipulation path, we need a set of additional assumptions. The idea
behind the construction of the algorithm is that from the uncountable
collection $\mc{E}$, we need to be able to extract from $\mc{E}$ a
\emph{countable} (possibly infinite) subcollection which still covers
the given path $\sigma$. Then from the countable subcollection, we can
then iterate through combinations of its members until we find one
that covers $\sigma$. The Heine-Borel covering theorem helps guarantee
that these iterations will eventually terminate in finite time.

Before we proceed to stating assumptions, we present the following
result. Define the set $\pc{F}(g)$ as the union of all element
$\pc{F}$ of $\mc{F}$ where the grasp associated with any composite
configuration $\bm{c} \in \pc{F}(g)$ is specified by the bimanual
grasp parameter vector $g$ (see Section~\ref{section:overview}). Since
there may exist multiple IK solutions associating with one grasp, we
may categorize the set $\pc{F}(g)$ further into a number of subsets
according to classes of the associated IK solution~\cite{Bur89ar,
  GCS09iros}. We write $\pc{F}(g, k), k \in K$ to refer to the set
$\pc{F}$ with a specific grasp $g$ and which any
$\bm{c} \in \pc{F}(g, k)$ has the IK solution in the same class as
other configurations. Note the according to~\cite{Bur89ar}, the index
set $K$ is bounded.

Consider a set $\pc{E}(g, k)$, defined as the projection via $\pi_p$
of $\pc{F}(g, k)$.

\begin{lemma}
  An object path $\sigma : [0, 1] \rightarrow SE(2)$ lies entirely in
  a connected component of $\pc{E}(g, k)$ if and only if it is a
  projection of a transfer path.
  \label{lemma:transfer}
\end{lemma}

\begin{proof}
  The result follows directly from the definition of a single-transfer
  connected component.
\end{proof}

\subsection{Assumptions}
Now we present a set of assumptions as follows.

\begin{assumption}
  For any object path $\sigma : [0, 1] \rightarrow SE(2)$. The
  intersection between $\sigma$ and the set $\pc{E}(g, k)$, for any
  grasp $g$ and IK class index $k$, consists of finitely many path
  segments and the domain of the path parameter for each segment is
  computable.
  \label{assumption:finite}
\end{assumption}

In usual manipulation planning settings, environments are relatively
controlled such that they should not contain physical obstacles of
extremely odd geometries which would eventually result in the set
$\pc{E}(g, k)$ being divided into infinitely many connected
components. Furthermore, the robot singularity set is not likely to
divide the feasible configuration space into infinitely many connected
components as well. This is true, for example, for a class of
\emph{generic} manipulators whose singularity sets consist of finite
\emph{smooth manifolds}~\cite{PL92tra}. However, the above assumption
is still necessary to ensure that each connected component of
$\pc{E}(g, k)$ is well-behaved, in the sense that a finite number of
components would not result in infinitely many segments.

The second assumption is stated as follows.

\begin{assumption}
  Given an object path $\sigma$ contained in a connected component of
  $\pc{E}(g)$, where $g = \left[g_1^\top g_2^\top\right]^\top$,
  $g_1 = [l_1\, a_1\, b_1\, \delta_1]^\top$, and
  $g_2 = [l_2\, a_2\, b_2\, \delta_2]^\top$. There exists a lower
  bound $\epsilon > 0$, which may depend on $\sigma$, such that all
  $\pc{E}(g')$ also contain $\sigma$, where
  $g' = \left[g_1'^\top g_2'^\top\right]^\top$,
  $g'_1 = [l_1\, a_1\, b_1\, (\delta_1 + \Delta_1)]^\top$,
  $g'_2 = [l_2\, a_2\, b_2\, (\delta_2 + \Delta_2)]^\top$, and
  $0 < \Delta_1, \Delta_2 \leq \epsilon$.
  \label{assumption:nearby_grasps}
\end{assumption}

This means that if the robots can grasp the object with the bimanual
grasp $g$ and then trace the object path $\sigma$. The robots can also
grasp the object and trace the same object path with some \emph{nearby
  grasps}.

Now suppose, without loss of generality, that the domain of the grasp
parameter $\delta$ is normalized to $(0, 1)$ for each grasp
class. Next, we define $\mc{B}(d)$, with
$d = (d_1, d_2) \in (0, 1) \times (0, 1)$, as the collection of all
$\pc{E}(g, k)$ such that the grasp parameters $\delta_1 = d_1$ and
$\delta_2 = d_2$. Consider the set $\mc{A}$ defined by
\begin{equation}
  \mc{A} = \big\{ \mc{B}(h_1),\, \mc{B}(h_2),\, \mc{B}(h_3),\, \ldots \big\},
\end{equation}
where $\{h_1, h_2, h_3, \ldots\}$ is a two-dimensional low-discrepancy
sequence, such as the sequence introduced in~\cite{LL03icra}. This set
$\mc{A}$ is then basically an enumerate of $\mc{B}(d)$ at values $d$
from such a sequence. Thanks to
Assumption~\ref{assumption:nearby_grasps}, we have that there exists
an integer $N$ such that the subset $\mc{A}'$ of the first $N$ terms
of $\mc{A}$ covers $\sigma$.

\begin{assumption}
  The connectivity of a set $\mc{T}_i$ for all $i$ can be determined
  empirically, e.g., by discretization.
\end{assumption}

The above discretization can be easily done on the three parameters
$(x, y, \theta)$ parameterizing the placement. Ranges of the
parameters $x$ and $y$ are determined by the user while $\theta$
ranges from $0$ to $2\pi$. After obtaining the set of discretized
coordinates, one tests at each point if the object is collision-free
and graspable by the robots. Fig.~\ref{figure:connectedness_scene}
shows the scene in which we tested the connectivity of $\mc{T}_i$. The
set $\mc{T}_i$ is visualized in Fig.~\ref{figure:connectedness_result}
by being superimposed into the scene.

The idea here is that once the connectivity of the set $\mc{T}_i$ is
determined, we can treat different connected components of $\mc{T}_i$
(if any) as different placement classes when computing a
certificate. Therefore, we shall suppose in the sequel that each
$\mc{T}_i$ consists of one connected component.

\begin{figure}
  \centering
  \subfloat[{}]{
    \includegraphics[width=0.2\textwidth]{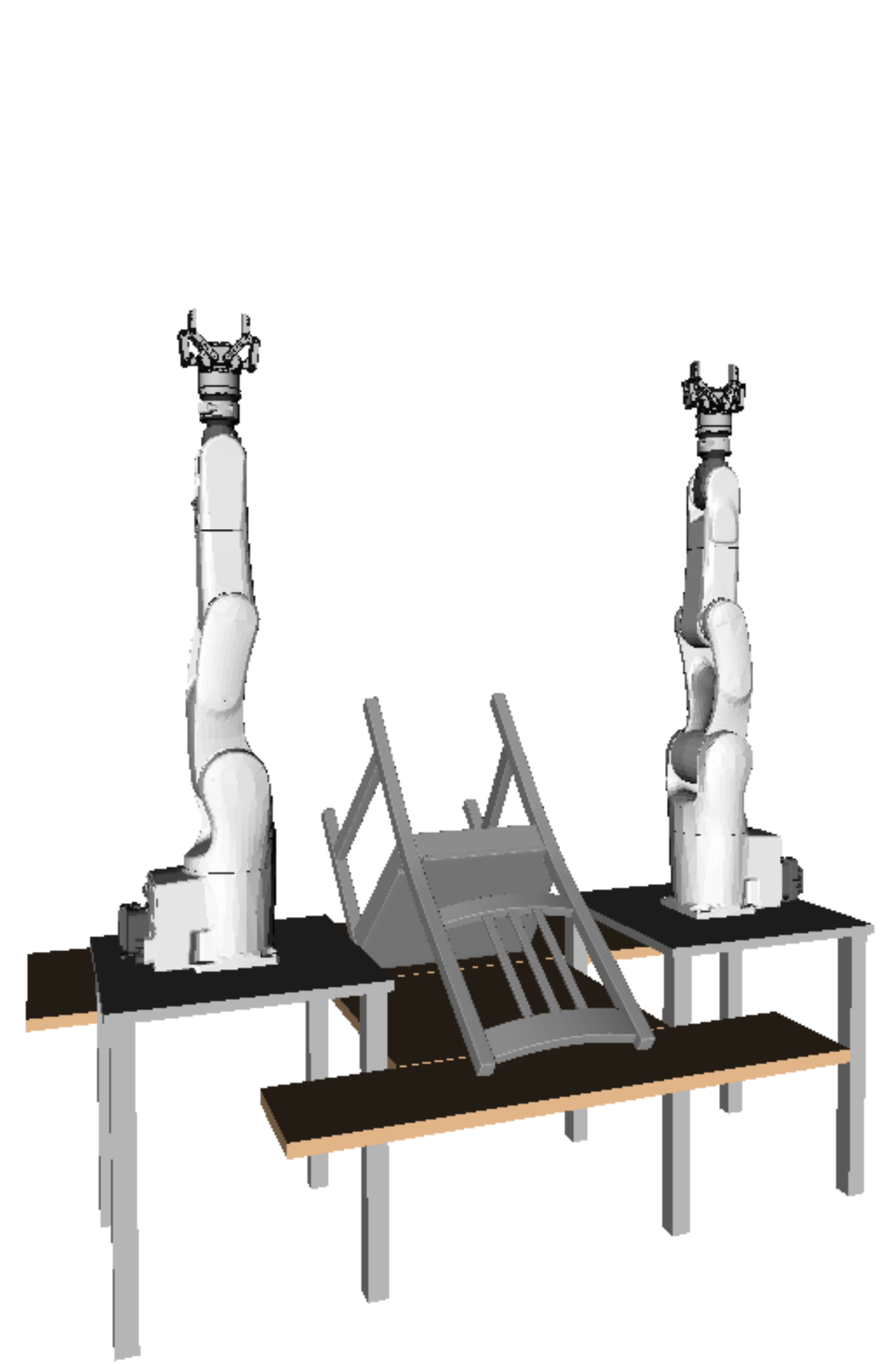}
    \label{figure:connectedness_scene}
  }
  \hspace{10pt}
  \subfloat[{}]{
    \includegraphics[width=0.2\textwidth]{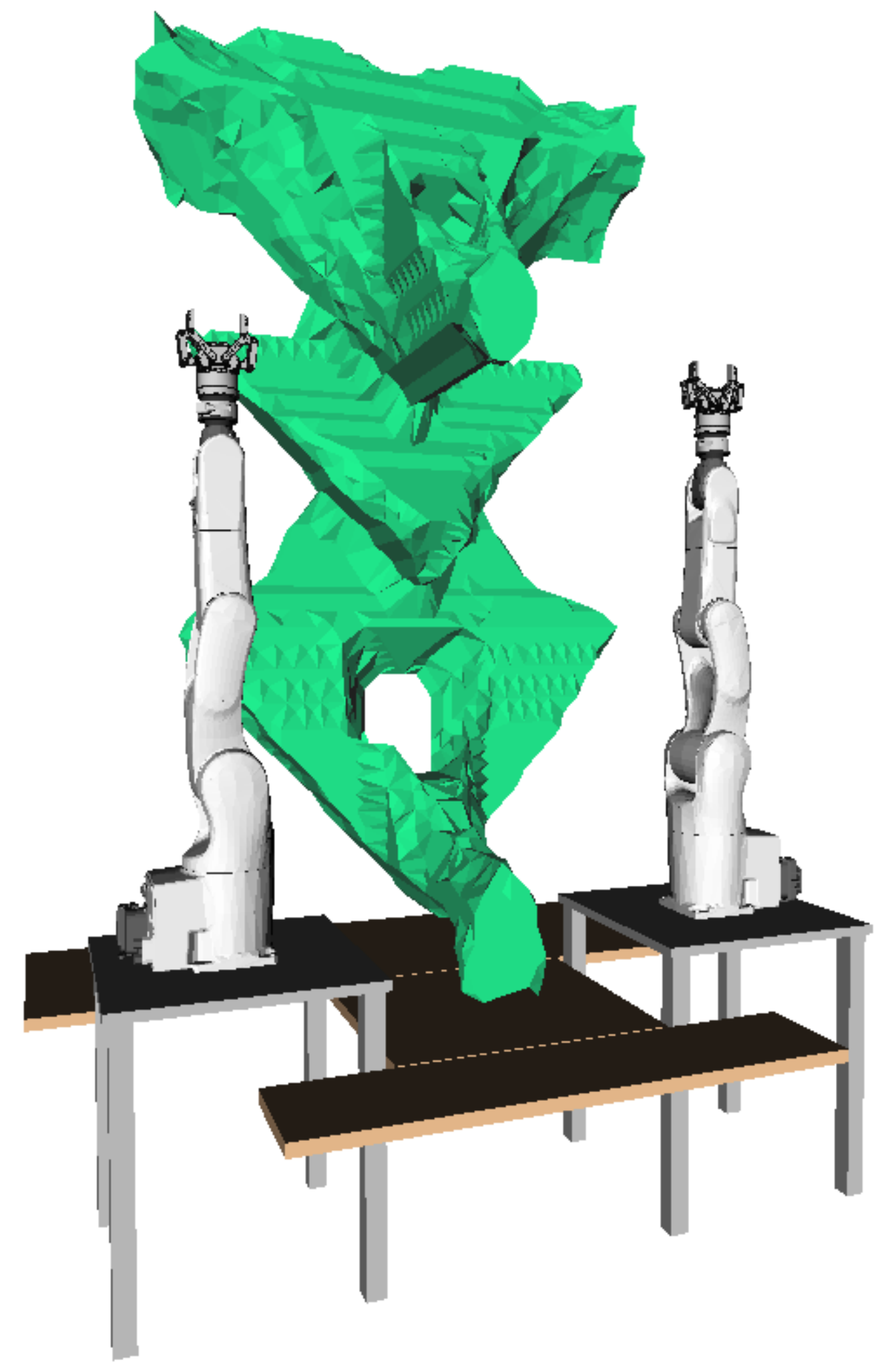}
    \label{figure:connectedness_result}
  }
  \caption{(a) The chair at a transformation in the placement class of
    interest. (b) The set $\mc{T}_i$ plotted out in 3D, shown in
    green. Each $3$D point represents $(x, y, \theta)$ at which the
    object can be grasp by the robots. The vertical axis is $\theta$,
    varying from $0$ to $2\pi$. One can see from the figure that
    $\mc{T}_i$ is indeed connected.}
  \label{figure:connectedness}
\end{figure}

\subsection{Algorithm}
Consider next the algorithm listed in
Algorithm~\ref{algorithm:in-placement}. To compute a \TA motion for to
move the object along a path segment $\sigma(s), s \in [t, t']$ (in
\func{ComputeCCMotion} line~\ref{algoline:track}), one starts with the
initial IK solution of the robot grasping the object at
$\sigma(t)$. Note that the grasp as well as the IK solution can be
determined uniquely from the element $A$. Since
$[t, t'] \subset (a, b)$, it is possible for one to use a differential
IK algorithm~\cite{SicX09book} to solve for remaining IK solutions
along the path (according to Lemma~\ref{lemma:transfer}).

Based on assumptions presented in the previous Section, we have the
following proposition.

\begin{proposition}
  Given an object path $\sigma$ lying entirely in $\mc{T}_i$,
  Algorithm~\ref{algorithm:in-placement} will terminate in finite time
  with a finite-length manipulation trajectory whose projection via
  $\pi_p$ is $\sigma$.
\end{proposition}

\begin{proof}
  Consider first the function \func{Planner} in
  Algorithm~\ref{algorithm:in-placement}. Since the set $\mc{A}$
  provides a countable open covers of the path $\sigma$, and since the
  closed interval $[0, 1]$ of the path parameter is compact, there
  exists a finite subcover of $\sigma$. Therefore, only a finite
  number of iterations is required before the while loop (in
  \func{Planner} line~\ref{algoline:whileloop}) terminates. In each
  iteration of the while loop, one call to the function
  \func{CheckCover} is made. From Assumption~\ref{assumption:finite}
  of finite intersections with $\sigma$, one would require finite time
  to verify intersections of each element $A$ of $\mc{A}'$ with the
  path $\sigma$. Therefore, \func{CheckCover} always terminates in
  finite time.

  Since $\mc{A}'$ is a finite subcover of the path $\sigma$, the while
  loop in \func{ComputeCCMotion} will eventually terminate. This
  concludes the proof.
\end{proof}

\begin{algorithm}[h]
  \caption{In-placement manipulation planner}
  \label{algorithm:in-placement}
  \Indm
  {\nonl{\func{Planner}($\sigma$):}}\;
  \Indp

  $\mc{A}' \gets \{\}, j \gets 1$\;

  \While{\upshape \kw{True}}{\label{algoline:whileloop}
    Append the $j^\text{th}$ element of $\mc{A}$ to $\mc{A}'$, $j$++\;
    \kw{covered} $\gets$ \func{CheckCover}$(\sigma, \mc{A}')$\;
    \If{\upshape \kw{covered}}{
      \textbf{break}\;
    }
  }

  \KwRet{\upshape \func{ComputeCCMotion}($\sigma, \mc{A}'$)}\;
  \vspace{5pt}
  \setcounter{AlgoLine}{0}
  \Indm
  {\nonl{\func{ComputeCCMotion}($\sigma, \mc{A}'$):}}\;
  \Indp

  $t \gets 0$, \kw{reached} $\gets$ \kw{False}\;

  \kw{traj} $\gets$ an empty trajectory\;
  
  \While{\upshape not \kw{reached}}{ 

    Select an element $A$ from $\mc{A}'$ which contains
    $\sigma(t)$. Let $(a, b)$ be the domain of $\sigma$ covered by
    $A$ such that $t \in (a, b)$\;

    Select an element $A'$ from $\mc{A}'$ which contains
    $\sigma(b)$. Let $(a', b')$ be the domain of $\sigma$ covered by
    $A'$ such that $b \in (a', b')$.

    Select $t' \in (a', b)$\;

    Compute a closed-chain motion for $\sigma(s), s \in [t, t']$ (and
    other necessary transit motions for
    regrasping)\label{algoline:track}\;

    Append the compute trajectories to \kw{traj}\;

    \If{$t' == 1$}{
      \textbf{break}\;
    }
    
    $t \gets t'$\;
  }

  \KwRet{\upshape \kw{traj}}\;
\end{algorithm}

\section{Generating Trajectories Between Different Placement Classes}
\label{section:external}

It follows from Proposition~\ref{proposition:existence} that if we
have one \TB trajectory which starts in some placement class
$\mc{P}_i$ and ends in some other placement class $\mc{P}_j$, any pair
of composite configurations in $\mc{P}_i \cup \mc{P}_j$ are also
manipulation path-connected. In the case when there are only two
placement classes available, any \TB path between the two placement
classes then guarantees the existence of a solution to any feasible
manipulation query. With $n_p$ placement classes available, one only
needs a minimum of $n_p - 1$ \TB trajectories between different pairs
of placement classes in order to guarantee the existence of a solution
to any feasible query. Therefore, we define the notion of a
\emph{certificate} as follows.

\begin{definition}
  A \emph{certificate} is a set of transfer paths that spans all the
  placement classes.
\end{definition}

One can think of placement classes as nodes in a graph. A certificate
is then analogous to a set of edges which contains a spanning tree's
edges. Although $n_p - 1$ transfer paths are sufficient to guarantee
the existence of solutions to any manipulation query, the more
transfer paths one has (between distinct pairs of placement classes)
can contribute to higher quality of solutions since the system may
need to visit a fewer number of intermediate placement classes before
reaching the desired placement class.

Since the process of computing a certificate needs to be done only
once per problem setting, we suppose that this computation can be done
offline and the computation time is not a limiting factor. Therefore,
one may aim at generating all $^{n_p}C_2 = n_p(n_p - 1)/2$ transfer
paths connecting all possible different pairs of placement classes.

Given a pair of placement classes $\mc{P}_i$ and $\mc{P}_j$, we divide
the process of generating a \TB trajectory into two main parts:
\begin{enumerate*}
\item generating a closed-chain query; and
\item solving a closed-chain query.
\end{enumerate*}

\subsection{Generating a Closed-Chain Query}
\label{section:ccquery_generation}
Randomly sampling two object transformations, one from each placement
class, may have a relatively low probability that the resulting
closed-chain query is solvable. This is mainly due to the fact that
the closed-chain constraint greatly reduces the range of motions of
the system. To deal with this issue, we propose a heuristic to help
generate closed-chain queries which are likely solvable. The idea
behind this is that since the bimanual manipulation system can exhibit
a very limited range of motions, queries should be generated such that
they intuitively do not require a large range of robot motions to
solve them.

Recall that object transformations in any placement class can be
parameterized by three parameters
(Section~\ref{section:internal}). The problem of generating
closed-chain queries is then boiled down to how one generates the
three parameters for the start and goal transformations. We first
define a \emph{manipulation point} $(x_m, y_m)$ on the support
surface. The position parameters $(x, y)$ of the transformations to be
generated will be assigned to be this point. Doing so greatly
simplifies query generation while not drastically reduce the
possibilities of the queries generated since the motion range of the
system is already very limited. The manipulation point may be assigned
to be on the middle line passing between the two robots (the green
line in Fig.~\ref{figure:flipping}), roughly speaking, to maximize the
reachability of the two robots.

Then the rotation parameter $\theta$ can be computed as follows. Let
$f_j$ denote the face of $\pc{H}$ corresponding to placement class
$\mc{P}_j$ (the goal placement class). Let $\bm{n}_j$ be a normal
vector pointing outwards from the face $f_j$ (see the red arrow in
Fig.~\ref{figure:ccquery_placement1}). The desired value of $\theta$
is such that the projection of $\bm{n}_j$ onto the support surface is
parallel to the line $l$. The reason behind this is that once the
object is arranged as mentioned, the expected closed-chain motion to
move the object from placement $\mc{P}_i$ to $\mc{P}_j$ will be a
relatively easy flipping motion, as shown by the blue arrow in
Fig.~\ref{figure:ccquery_placement1}. After the start transformation
has been computed, the goal transformation can be computed accordingly
(Fig.~\ref{figure:ccquery_placement2}).

\begin{figure}
  \centering
  \subfloat[{}]{
    \includegraphics[width=0.2\textwidth]{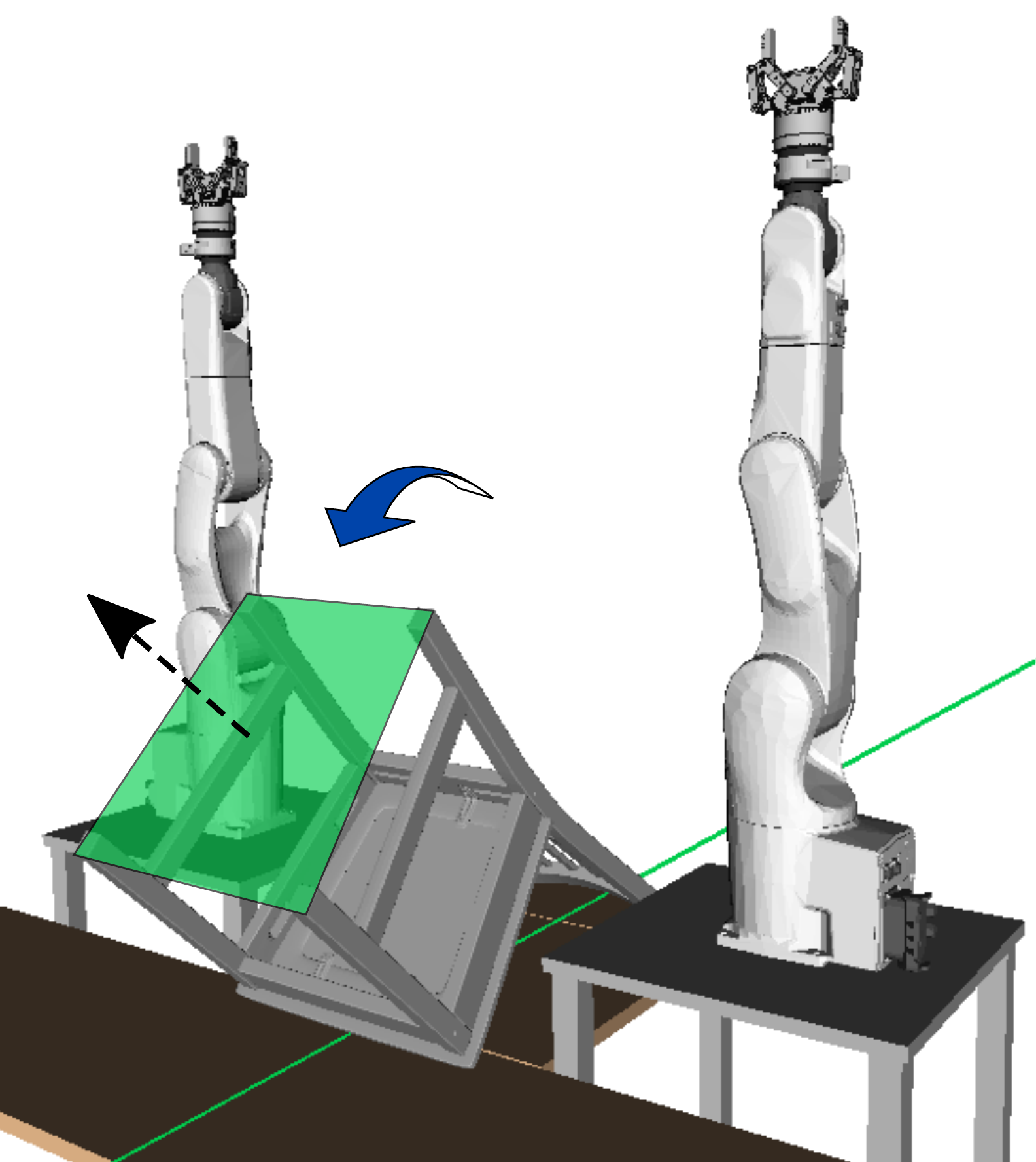}
    \label{figure:ccquery_placement1}
  }
  \hspace{10pt}
  \subfloat[{}]{
    \includegraphics[width=0.2\textwidth]{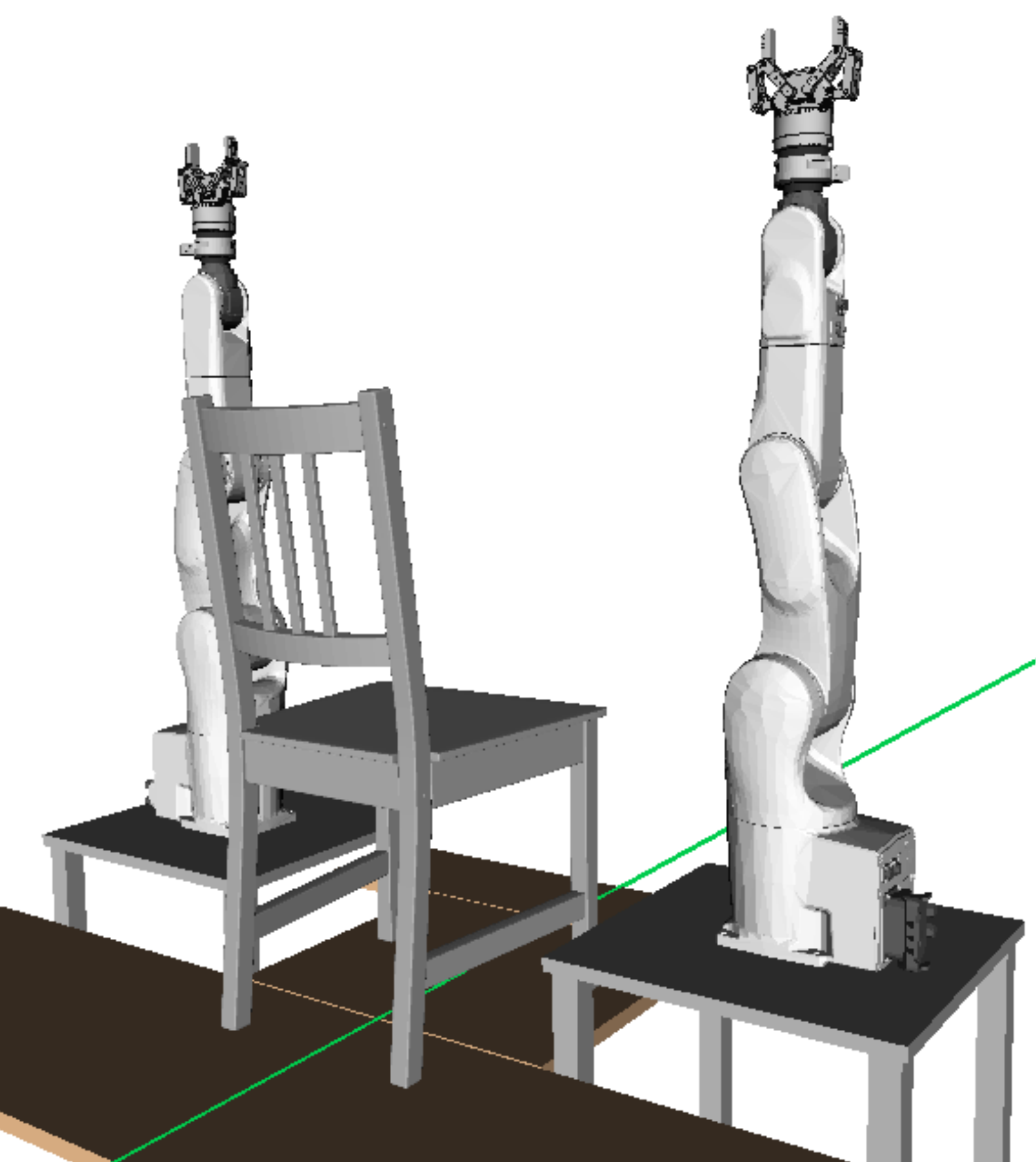}
    \label{figure:ccquery_placement2}
  }  
  \caption{(a) The start transformation of the chair. The chair is
    placed around the center line (green) such that the normal vector
    (the dotted arrow) of the face $f_j$ is pointing in the direction
    of the center line. By arranging the chair as shown, the chair is
    only needed to be flipped in the direction depicted by the curved
    arrow. (b) The goal transformation of the chair.}
  \label{figure:flipping}
\end{figure}

Apart from the two transformations $\bm{T}_s$ and $\bm{T}_g$, we may
also include into the query a grasp $g$ together with associated IK
solutions at $\bm{T}_s$ and/or $\bm{T}_g$. Generating a grasp $g$ is
straightforward since it can be sampled from a set of grasps available
at both transformations.

Computing associated IK solutions at $\bm{T}_s$ and/or $\bm{T}_g$,
however, is non-trivial when the query is to be solved via a
bidirectional planner. This is because given two composite
configurations in the same grasp class, there is no known way to
completely determine if they belong to the same connected component of
$\mc{G}$. Generally, one set of IK solutions of robots grasping the
object corresponds to one connected component of $\mc{G}$ called
self-motion manifold~\cite{Bur89ar}. To choose IK solutions of robots
grasping the object at both $\bm{T}_s$ and/or $\bm{T}_g$, we rely on
an ad hoc heuristic since, to our knowledge, there is currently no
known generalized way to do so.

Finally, note that this two-stage approach in generating and solving
queries may proceed in iterations. If a generated query is not
solvable within the given time, one can generate a new query by
defining a new manipulation point, e.g., by adding some small
perturbation to the point.


\begin{figure*}
  \centering \subfloat[Object
  $1$]{\includegraphics[width=0.14\textwidth]{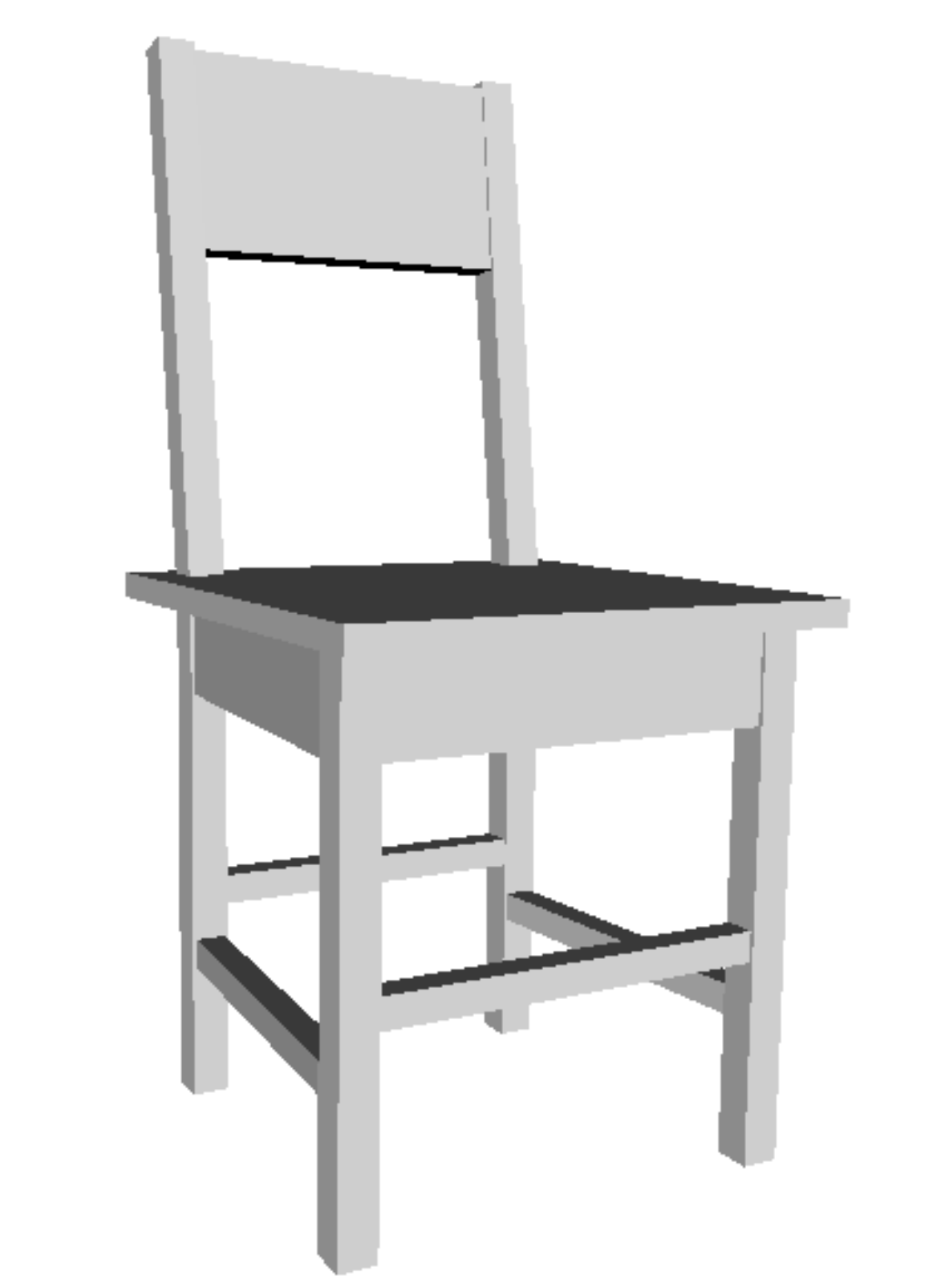}}
  \hspace{8pt} \subfloat[Object
  $2$]{\includegraphics[width=0.14\textwidth]{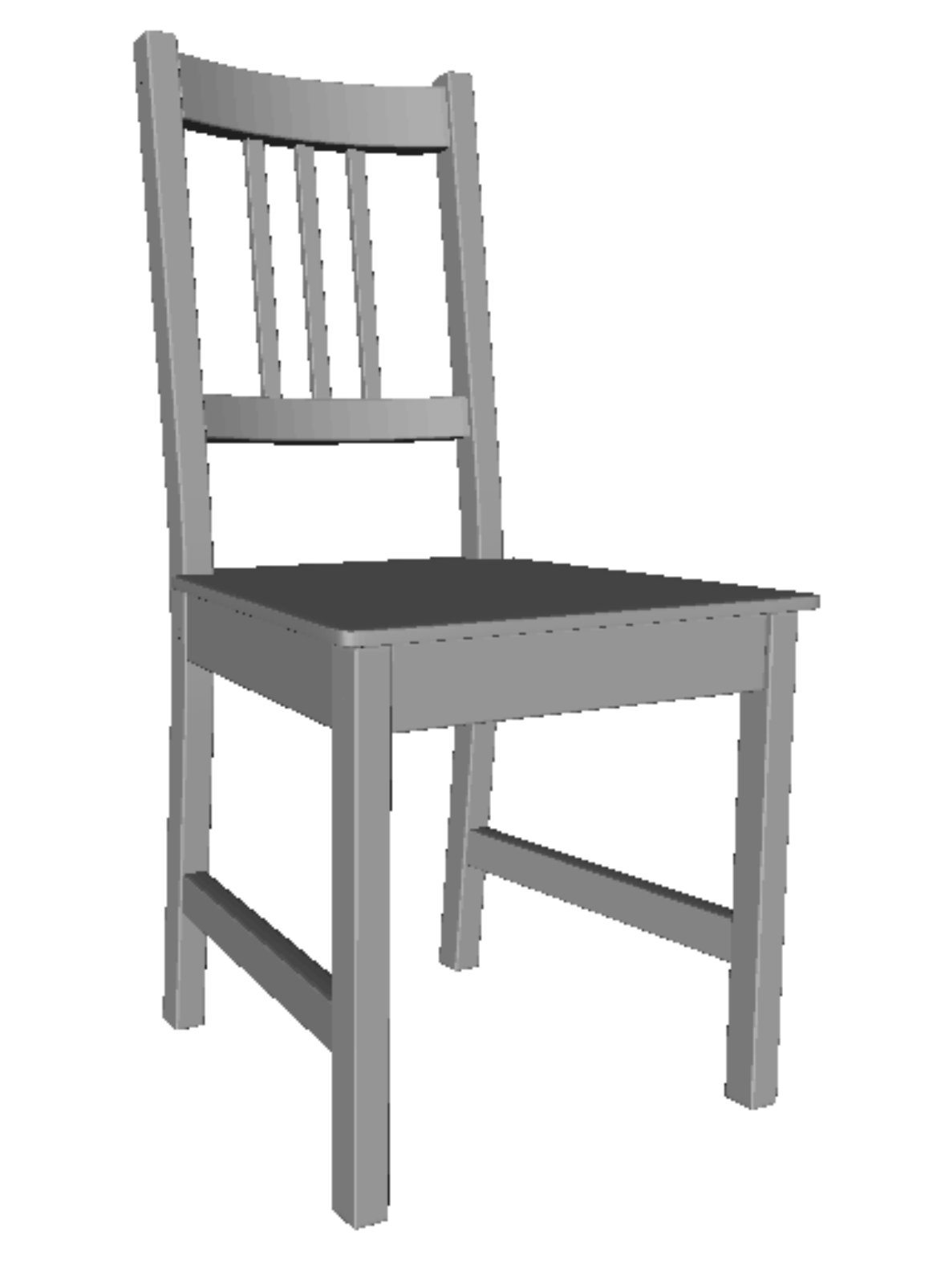}\label{figure:stefan}}
  \hspace{8pt} \subfloat[Object
  $3$]{\includegraphics[width=0.14\textwidth]{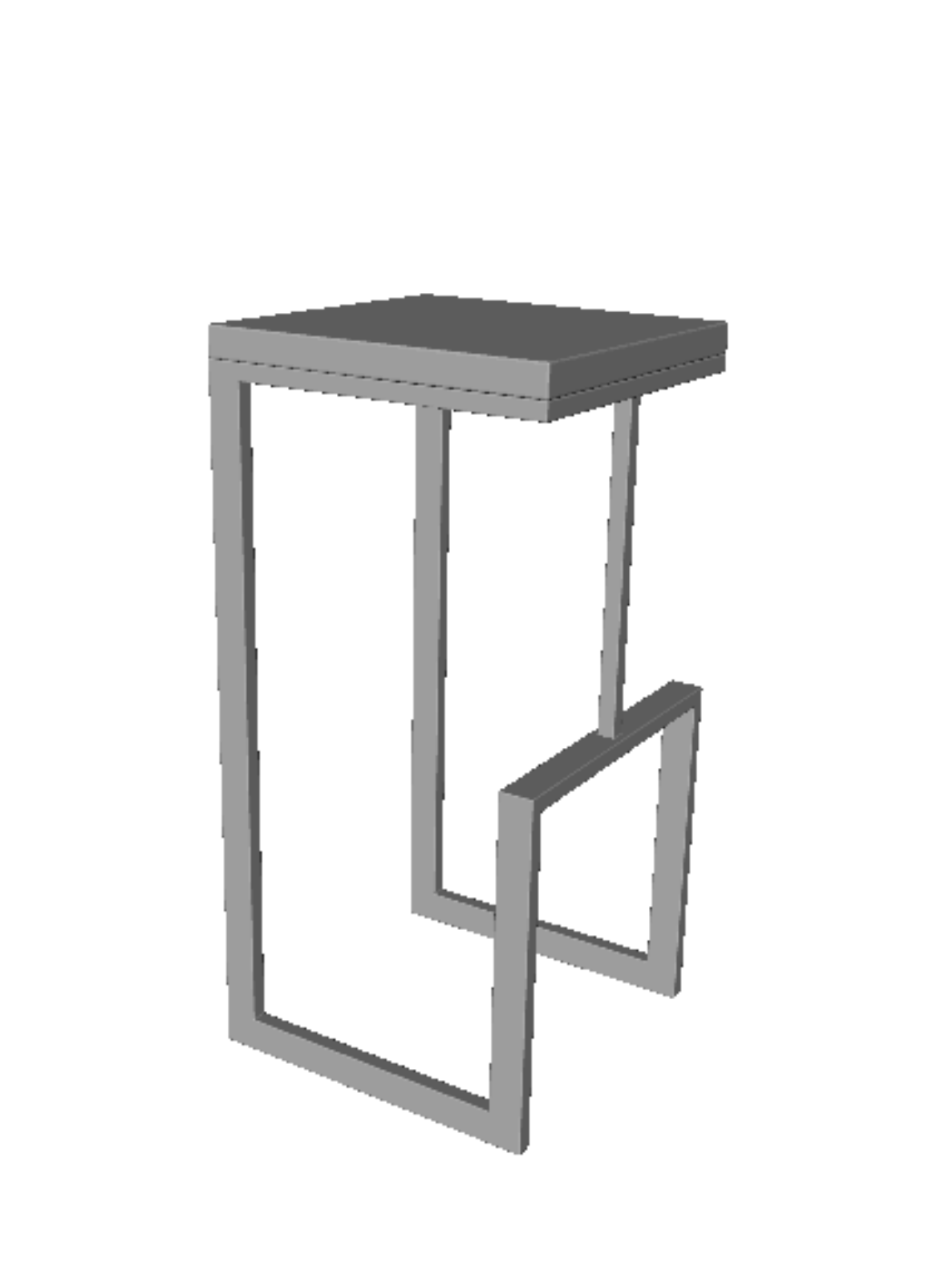}}
  \hspace{8pt} \subfloat[Object
  $4$]{\includegraphics[width=0.14\textwidth]{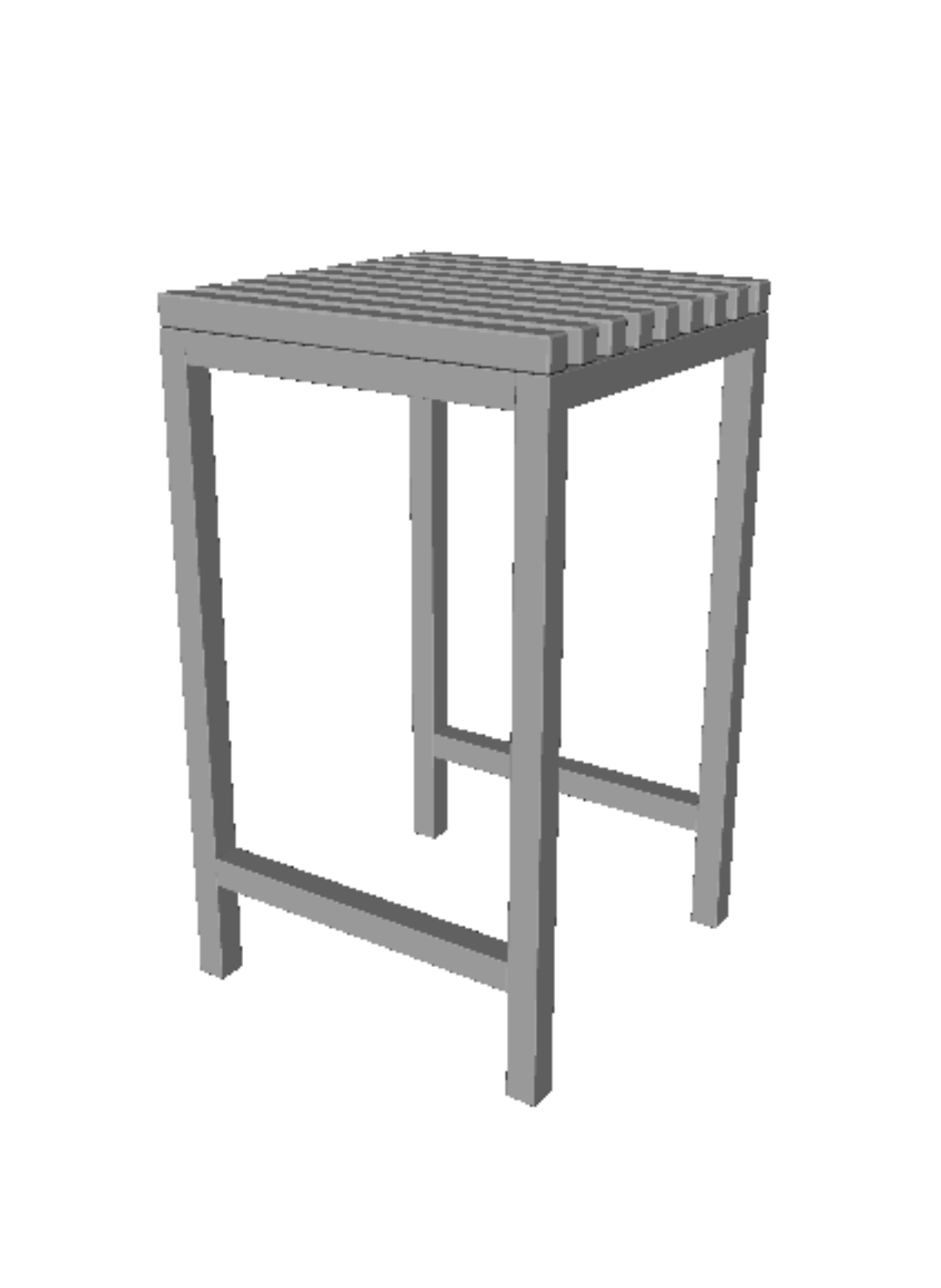}}
  \hspace{8pt} \subfloat[Object
  $5$]{\includegraphics[width=0.14\textwidth]{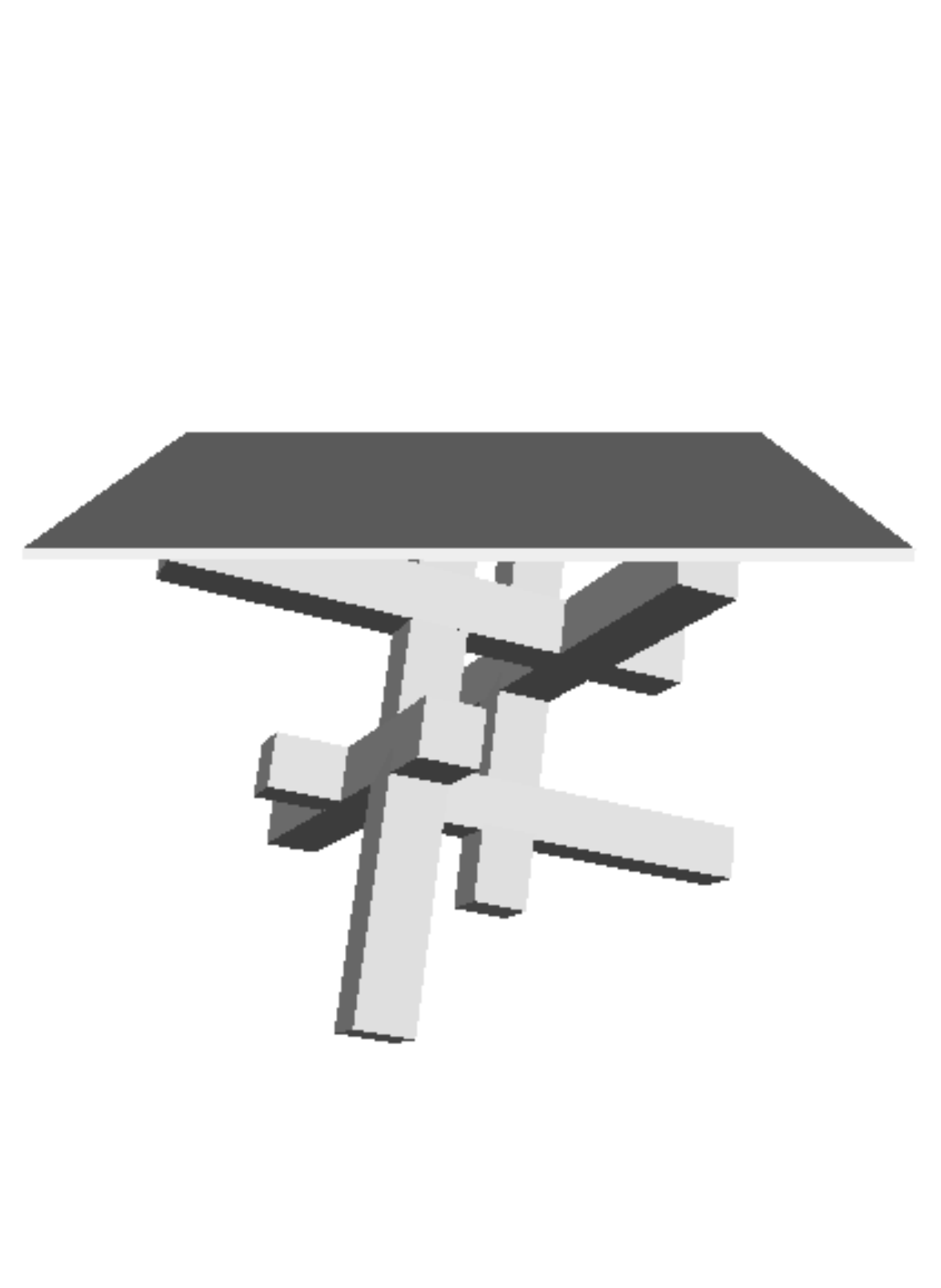}}
  \hspace{8pt} \subfloat[Object
  $6$]{\includegraphics[width=0.14\textwidth]{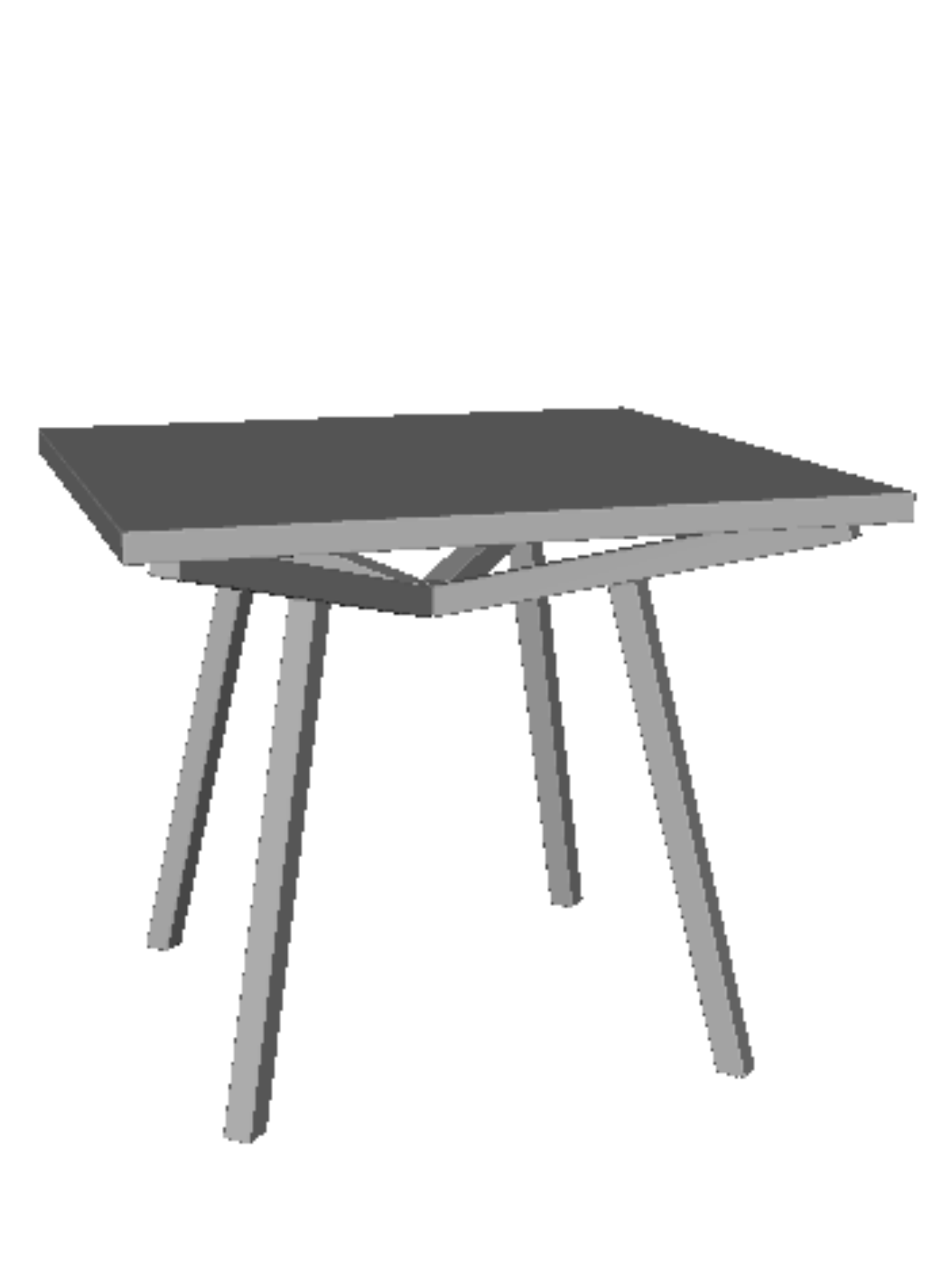}}
  \caption{Objects used in our experiments of certificate computation (not to scale).}
  \label{figure:objects}
\end{figure*}

\begin{table*}
  \centering
  \caption{Computation Time for Computing Certificates and Solving
    Queries for the Objects in Fig.~\ref{figure:objects}.}
  \label{table:results}
  \renewcommand{\arraystretch}{1.3}
  \begin{tabularx}{1.00\textwidth}
    {
      @{\extracolsep{\fill}}|
      >{\setlength\hsize{0.05\hsize}\centering}X|
      >{\setlength\hsize{0.20\hsize}\centering}X|
      >{\setlength\hsize{0.10\hsize}\centering}X|
      >{\setlength\hsize{0.15\hsize}\centering}X|
      >{\setlength\hsize{0.15\hsize}\centering}X|
      >{\setlength\hsize{0.15\hsize}\centering}X|
      >{\setlength\hsize{0.15\hsize}\centering}X|
      >{\setlength\hsize{0.15\hsize}\centering}X|
      >{\setlength\hsize{0.15\hsize}\centering}X|
    }
    \cline{4-9}

    \multicolumn{3}{c|}{} &
    Object $1$ & Object $2$ & Object $3$ & Object $4$ & Object $5$ & Object $6$
    \tabularnewline
    \hline

    \multicolumn{3}{|c|}{\# placement classes} &
    $4$  & $3$ & $6$ & $6$ & $8$ & $6$
    \tabularnewline
    \hline

    \multicolumn{3}{|c|}{\# grasp classes} &
    $1849$ & $2601$ & $2401$ & $3364$ & $1089$ & $784$
    \tabularnewline
    \hline
    \hline

    \multirow{7}{*}[0.0em]{
      \hspace{-1pt}
      \parbox[t]{30pt}{
        \rotatebox[origin=c]{90}{Computing Certificates}
      }
    } &
    \multirow{2}{*}[0.0em]{
      \renewcommand{\arraystretch}{1.0}
      \begin{tabular}{@{}c@{}}
        total\\
        computation time
      \end{tabular}
    } &
    avg. (s.) &
    $415.97$ & $281.18$ & $867.71$ &
    $1030.44$ & $1374.27$ & $574.54$
    \tabularnewline
    \cline{3-9}

    \multicolumn{1}{|c|}{} & \multicolumn{1}{c|}{} &
    std. (s.) & $141.49$ & $98.48$ & $212.54$ & $288.06$ & $342.11$ & $120.29$
    \tabularnewline
    \cline{2-9}

    \multicolumn{1}{|c|}{} &
    \multicolumn{2}{c|}{*planning (s.)} &
    $154.36$ & $94.07$ & $604.19$ & $556.92$ & $582.54$ & $199.63$
    \tabularnewline
    \cline{2-9}

    \multicolumn{1}{|c|}{} &
    \multicolumn{2}{c|}{*grasp check (s.)} &
    $101.69$ & $76.27$ & $202.90$ & $227.51$ & $157.90$ & $154.11$
    \tabularnewline
    \cline{2-9}

    \multicolumn{1}{|c|}{} &
    \multicolumn{2}{c|}{*shortcutting (s.)} &
    $14.51$ & $7.29$ & $32.76$ & $33.55$ & $34.83$ & $23.58$
    \tabularnewline
    \cline{2-9}

    \multicolumn{1}{|c|}{} &
    \multicolumn{2}{c|}{\# \TB trajectories (avg.)} &
    $4.00$ & $2.00$ & $11.02$ & $10.84$ & $11.98$ & $7.40$ 
    \tabularnewline
    \cline{2-9}

    \multicolumn{1}{|c|}{} &
    \multicolumn{2}{c|}{success rate} &
    $1.0$ & $1.0$ & $1.0$ & $1.0$ & $1.0$ & $1.0$
    \tabularnewline
    \hline
    \hline

    \multirow{3}{*}[0.0em]{
      \vspace{-5pt}
      \hspace{-7pt}
      \parbox[t]{50pt}{
        \rotatebox[origin=c]{90}{
          \renewcommand{\arraystretch}{1.0}
          \begin{tabular}{@{}c@{}}
            Solving\\
            Queries
          \end{tabular}
        }
      }
    } &
    \multirow{2}{*}[0.0em]{
      \renewcommand{\arraystretch}{1.0}
      \begin{tabular}{@{}c@{}}
        total\\
        computation time
      \end{tabular}
    } &
    avg. (s.) & $103.65$ & $45.08$ & $127.45$ & $234.54$ & $109.99$ & $185.08$
    \tabularnewline
    \cline{3-9}

    \multicolumn{1}{|c|}{} & \multicolumn{1}{c|}{} &
    std. (s.) & $18.04$ & $14.93$ & $29.93$ & $55.96$ & $24.37$ & $31.55$
    \tabularnewline
    \cline{2-9}

    \multicolumn{1}{|c|}{} &
    \multicolumn{2}{c|}{\# \TB trajectories (avg.)} &
    $2.0$ & $2.0$ & $2.0$ & $2.0$ & $2.0$ & $2.0$
    \tabularnewline
    \hline

  \end{tabularx}
  \vspace{3pt}
  \caption*{\footnotesize *Row $3$ -- $5$ of ``Computing Certificate''
    category are planning time, grasp checking time, and shortcutting
    time averaged over successful closed-chain queries.}
\end{table*}


\subsection{Solving a Closed-Chain Query}

Our closed-chain motion planner (CCPlanner) is adapted from a
bidirectional RRT planner~\cite{LK01acr}. In particular, we build two
trees $\pc{T}_a$ and $\pc{T}_b$, each one is a data structure storing
vertices. A vertex $\pc{V}$ keeps information of a composite
configuration, its parent on the tree, as well as a closed-chain
trajectory connecting itself and its parent. The algorithm for
CCPlanner is listed in Algorithm~\ref{algorithm:ccplanner}.

\begin{algorithm}
  \caption{Closed-chain motion planner}
  \label{algorithm:ccplanner}
  \Indm
  {\nonl{\func{CCPlanner}($\bm{T}_s$, $\bm{T}_g$,
      $(\bm{q}_{1s}, \bm{q}_{2s})$, $(\bm{q}_{1g}, \bm{q}_{2g}$):}}\;
  \Indp
  
  $\pc{T}_a \leftarrow$ \func{InitializeTree}($\bm{T}_s$,
  $(\bm{q}_{1s}, \bm{q}_{2s})$)\;
  
  $\pc{T}_b \leftarrow$ \func{InitializeTree}($\bm{T}_g$,
  $(\bm{q}_{1g}, \bm{q}_{2g})$)\;
  
  \While{\upshape time is not exhausted}{
    $\bm{T} \leftarrow$ \func{SampleSE3}()\;
    \If{\upshape \func{Extend}($\pc{T}_a, \bm{T}$)}{
      \If{\upshape \func{Connect}($\pc{T}_a, \pc{T}_b$)}{
        \KwRet{\upshape \func{ExtractTrajectory($\pc{T}_a, \pc{T}_b$)}}\;
      }
    }
    \func{Swap}($\pc{T}_a, \pc{T}_b$)\;
  }
  \KwRet{\upshape \kw{None}}\;
  \vspace{5pt}
  \setcounter{AlgoLine}{0}
  \Indm
  {\nonl{\func{Extend}($\pc{T}, \bm{T}$):}}\;
  \Indp
  \For{\upshape $\subt{\pc{V}}{near}$ in \func{KNN}($\pc{T}, \bm{T}$)}{
    
    $\subt{\bm{T}}{new} \leftarrow$
    \func{Threshold}($\subt{\pc{V}}{near}$, $\bm{T}$)\;
    
    $\pc{P} \leftarrow$
    \func{Interpolate}($\subt{\pc{V}}{near}$.\kw{config},
    $\subt{\bm{T}}{new}$)\;
    
    \If{\upshape $\pc{P}$ is not \kw{None}}{
      $\subt{\pc{V}}{new} \leftarrow$ \func{Vertex}($\bm{c}, \pc{P}$)\;
      $\pc{T}$.\func{Add}($\subt{\pc{V}}{near}, \subt{\pc{V}}{new}$)\;
      \KwRet{\upshape \kw{True}}\;
    }
  }
  \KwRet{\upshape \kw{False}}\;
\end{algorithm}

CCPlanner accepts the start and goal transformations, $\bm{T}_s$ and
$\bm{T}_g$ and IK solutions at $\bm{T}_s$ and $\bm{T}_g$ as its
inputs. It starts by initializing two trees, \ie, creating root
vertices storing configuration information. In each planning
iteration, CCPlanner randomly samples an object transformation matrix
$\bm{T}$. Then the planner tries to extend a tree towards it. Upon a
successful extension the planner tries to connect two trees
together. If the connection attempt is successful, the closed-chain
path is extracted from the tree and returned. Some key functions are
described in details below.

\begin{itemize}[leftmargin=*, itemindent=0.5em]
\item \func{SampleSE3}: A transformation matrix (an element of
  $SE(3)$) is generated by separately sampling rotational and
  translational parts. A rotation matrix is uniformly sampled from
  $SO(3)$ via the method proposed in~\cite{Kuff04icra}. A translation
  vector is sampled uniformly from the user defined range.
  
\item \func{Extend}: To extend a tree towards a given transformation
  $\bm{T}$, we search in the tree the set of $k$ vertices whose
  transformation matrices are nearest to $\bm{T}$ (via \func{KNN} with
  $k$ defined by the user). The distance metric used is defined in
  Section~\ref{section:overview}. Then for each vertex, we generate
  a new transformation $\subt{\bm{T}}{new}$ in the direction from
  $\subt{\pc{V}}{new}$.\kw{config}.$\bm{T}$ to $\bm{T}$ such that the
  distance between $\subt{\bm{T}}{new}$ and
  $\subt{\pc{V}}{new}$.\kw{config}.$\bm{T}$ does not exceed some
  pre-defined step size. Then we construct a closed-chain trajectory
  connecting the two transformation via \func{Interpolate}. If the
  trajectory generation is successful, a new vertex,
  $\subt{\pc{V}}{new}$, is created and added to the tree.
  
\item \func{Interpolate}: We generate an $SE(3)$ trajectory first, by
  using the method in~\cite{PR97tg} for the rotational part and using
  polynomial interpolation for the translational part. Then the
  trajectory is discretized into small time steps. IK solutions of the
  two robots are computed at each time step using the differential IK
  method~\cite{SicX09book}.
  
  Apart from being less complicated implementation-wise, this method
  gives an exact parameterization of the object trajectory. One can
  then incorporate various types of constraints into the object
  trajectory by means of time-parameterization to obtain time-optimal
  trajectory with respect to the constraints (see~\cite{Pha14tro} for
  more details on time-optimal path parameterization (TOPP)). Examples
  of applicable constraints are velocity and acceleration limits for
  rigid body motions~\cite{HP16jgcd} and dynamic grasp stability.

  Nevertheless, the user can utilize their trajectory generation
  method of choice. In case an exact parameterization of a
  closed-chain trajectory is available, one may also use the TOPP
  method for redundantly-actuated systems~\cite{PS14mech} (which is
  the case for bimanual systems) to enforce the aforementioned
  constraints along the trajectory.
  
\item \func{Connect}: After a successful tree extension, the planner
  will attempt to connect the newly added vertex to the other
  tree. The details of procedure are mainly similar to \func{Extend},
  except this function does not include \func{Threshold}.
\end{itemize}

\section{Experimental Results}
\label{section:experiments}

The planner and all related functions were implemented in Python. We
used OpenRAVE~\cite{Dia10these} as a simulation environment. The
robots were two identical $6$-DOF industrial manipulators Denso
VS-$060$. Each one was equipped with a $2$-finger Robotiq gripper
$85$. The planning environment is as shown in
Fig.~\ref{figure:connectedness}. All simulations were run on a $2.4$
GHz desktop.

\subsection{Computing Certificates}
To validate our certificate issuing planner, we ran the planner to
compute certificates for a set of objects. All objects, listed in
Fig.~\ref{figure:objects}, were furniture pieces which were relatively
large such that they needed two robots grasping in order to move them.
For each object, we repeated certificate generation for $50$
times. When computing each transfer path, we also had the following
additional computations:
\begin{enumerate*}
\item grasp equilibrium checking
\item closed-chain trajectory shortcutting ($200$ iterations).
\end{enumerate*}
Statistics collected from the runs are reported in
Table~\ref{table:results} in ``Computing Certificate'' category. Here
are a few things we would like to point out:
\begin{itemize}[leftmargin=*, itemindent=0.5em]
\item The planner may spend up to around $45\%$ of the total time
  generating and planning unsuccessful queries. This is because there
  currently exists no definite method to check that the generated IK
  solutions associated with the start and goal transformations belong
  to the same connected component (self-motion manifold).
  
\item Due to the large number of grasp classes, solving a bimanual
  manipulation query via, for example, the three-dimensional extension
  of the high-level Grasp-Placement Graph~\cite{LP15ral} could
  potentially be rendered infeasible. Consider Object $2$
  (Fig.~\ref{figure:stefan}), for example, which has $3$ placement
  classes and $51$ \emph{unimanual} grasp classes. While the
  two-dimensional Grasp-Placement Graph has $104$ vertices and $2028$
  edges, its three-dimensional counterpart contains over $3,000$
  vertices with over $6$ million edges.
  
\item For each run, the certificate computation procedure is
  considered successful if the computed transfer paths span all the
  placement classes. Although the number of successfully planned \TB
  paths varied slightly among different runs, the resulting set of \TB
  paths still spanned all placement classes in every run.
  
\item In the current implementation, we check grasp equilibrium by
  solving a linear program at each discretized time step along a
  closed-chain trajectory. This approach is, however, time-consuming
  and restrictive in that it only guarantees static
  equilibrium\footnote{Trajectories with static equilibrium are only
    guaranteed to be executable at \emph{arbitrarily slow}
    speed.}. One possible improvement is by formulating contact
  constraints in terms of inequalities in path parameters and its
  derivatives~\cite{CPN15rss} by utilizing cone double-description
  (CDD) method~\cite{FP96ccs}. Then one can \emph{time-optimally}
  parameterize the trajectory such that it moves as fast as possible
  while respecting all the constraints (see~\cite{Pha14tro} and
  references therein for more details).
\end{itemize}

\subsection{Solving Bimanual Queries}
First, for each object listed in Fig.~\ref{figure:objects}, we
hand-pick two placement classes which do not have any direction
connection via a \TB trajectory (information provided by a
certificate) and generated a pair of object transformations $\bm{T}_s$
and $\bm{T}_g$ from each of the placement classes. Then we repeat
solving each query $Q = (\bm{T}_s, \bm{T}_g)$ for $50$
times. Statistics collected from the runs are reported in
Table~\ref{table:results} in ``Solving Queries'' category. One of the
main factors which causes variations in the running time is the
geometry of each object. Larger objects, for example, have
\emph{narrower} free space to navigate on the support
surface. Furthermore, with larger objects, it is also more difficult
for robots to move around and change grasps.

A solution manipulation path to any of the above queries needs to be
at least of length $5$. However, it is not the case here since a \TA
trajectory, which connects two \TB trajectories from a certificate,
will always contain regrasp operations. This is because the grasps
used in the two \TB trajectories are always different. Therefore, if a
planner with an extension of the high-level Grasp-Placement Graph is
used, it will spend a considerable amount of time invalidating
manipulation paths of shorter length, hence not practical. A general
manipulation planner such as a Random-MMP~\cite{HN11ijrr} is not
likely to terminate with a solution within a reasonable amount of time
as well since it needs to sample correctly a relative long sequence of
transit and transfer.

Apart from the simulation results, we also successfully carried out a
hardware experiment. We constructed a query for Object $2$
(Fig.~\ref{figure:stefan}) such that there is no direct \TB trajectory
connecting the two placement classes. The scene with the start and
goal transformations of the object provided by the query are shown in
Fig.~\ref{figure:Tstart} and Fig.~\ref{figure:Tgoal},
respectively. The computed solution to this query consists of three
\TA trajectories ((a)$\rightarrow$(b), (c)$\rightarrow$(d), and
(e)$\rightarrow$(f)) and two \TB trajectories ((b)$\rightarrow$(c) and
(d)$\rightarrow$(e)).

The controller used in this experiment was similar to the one
presented in~\cite{XLP16arxiv}. The video of the robots executing the
motion solving this query (on the real platform) can be found at
{\tt{https://youtu.be/4DcMwr2xxrQ}}.

\begin{figure}[h]
  \centering
  \subfloat[{}]{
    \includegraphics[width=0.15\textwidth]{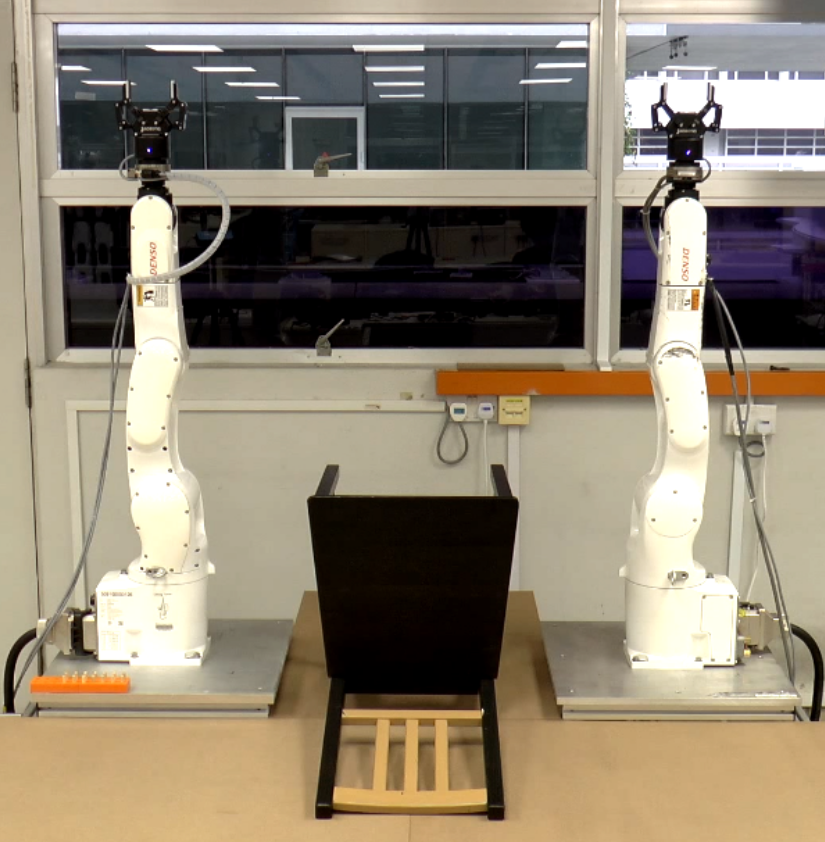}
    \label{figure:Tstart}
  }
  \subfloat[{}]{
    \includegraphics[width=0.15\textwidth]{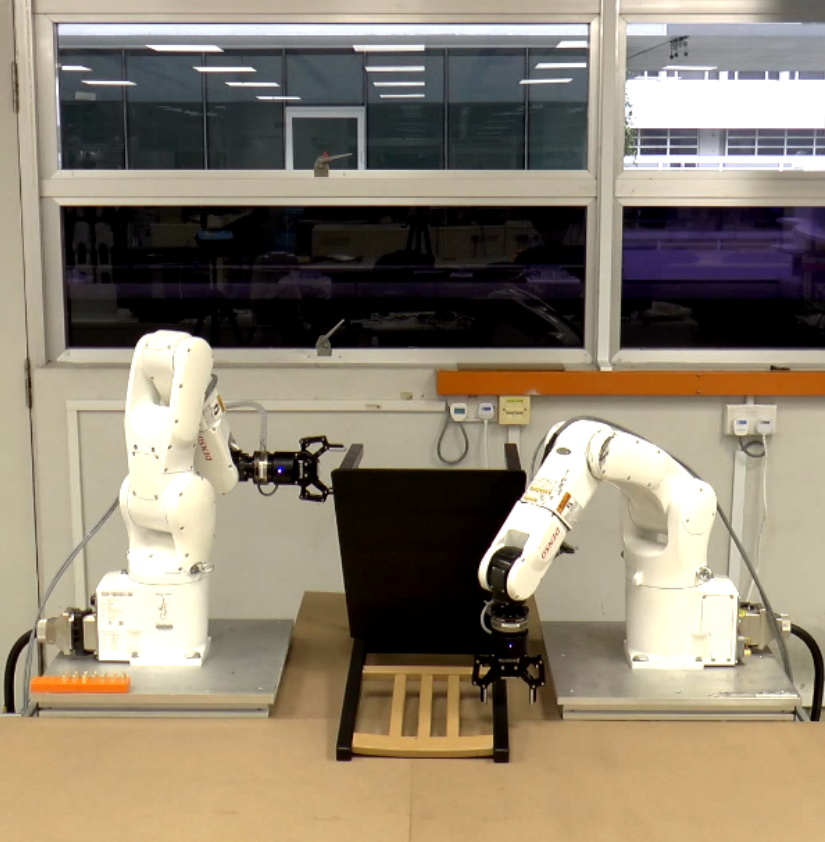}
    \label{figure:Tint1}
  }
  \subfloat[{}]{
    \includegraphics[width=0.15\textwidth]{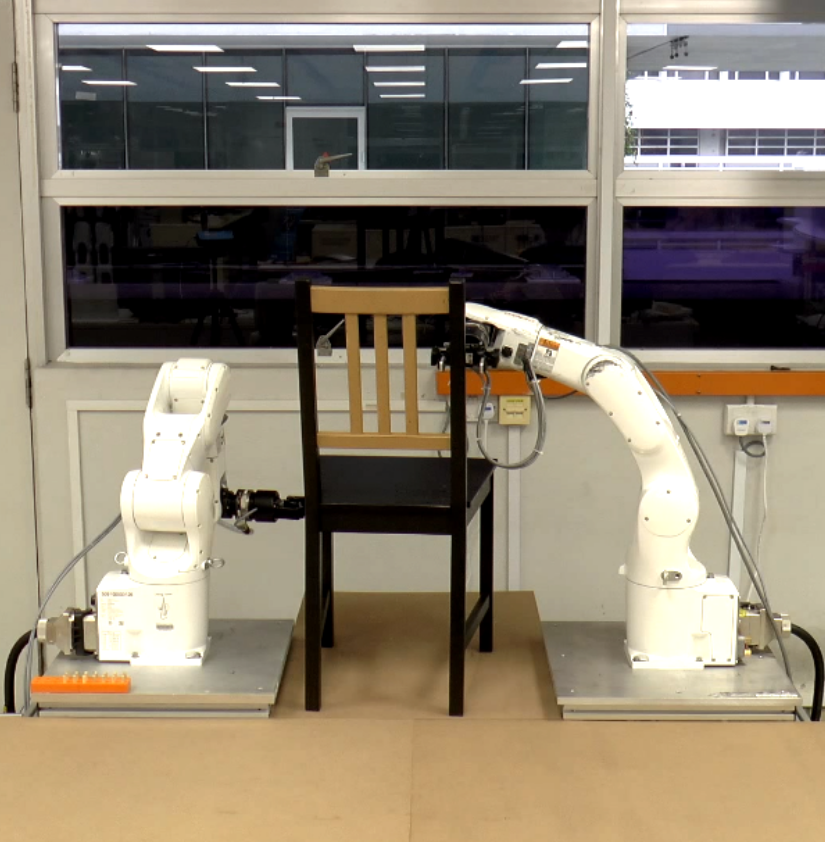}
    \label{figure:Tint2}
  }
  \\
  \subfloat[{}]{
    \includegraphics[width=0.15\textwidth]{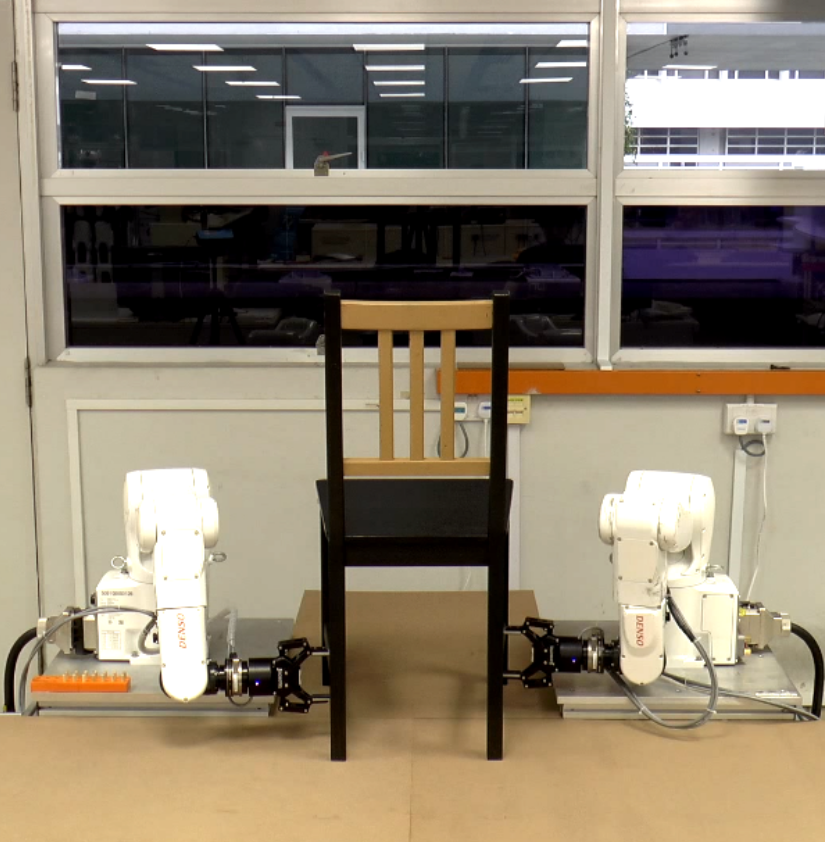}
    \label{figure:Tint3}
  }
  \subfloat[{}]{
    \includegraphics[width=0.15\textwidth]{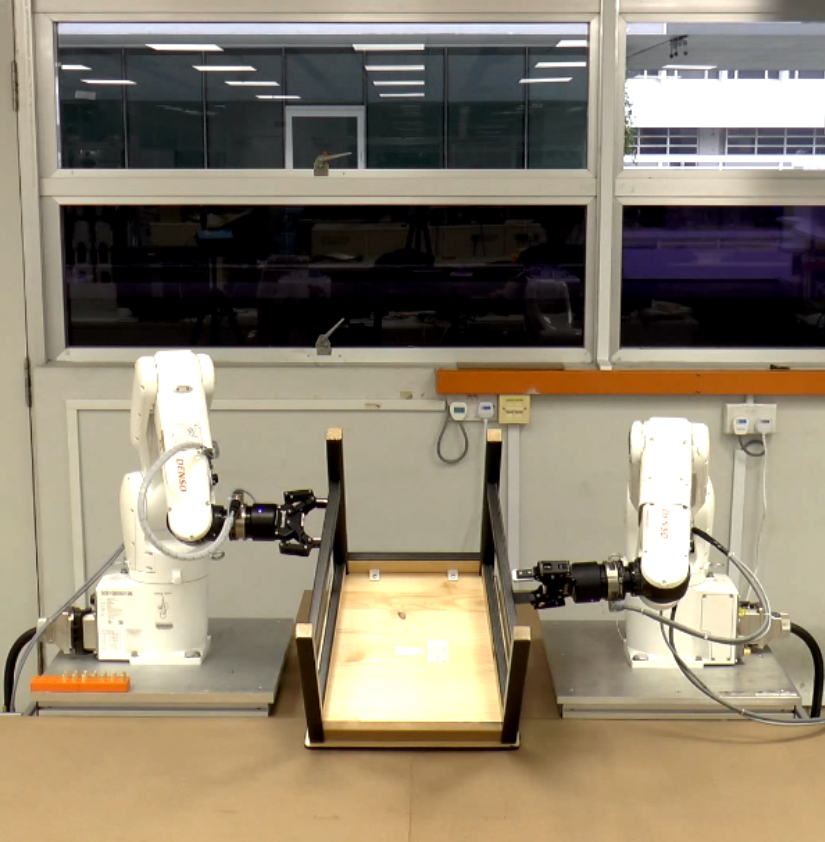}
    \label{figure:Tint4}
  }
  \subfloat[{}]{
    \includegraphics[width=0.15\textwidth]{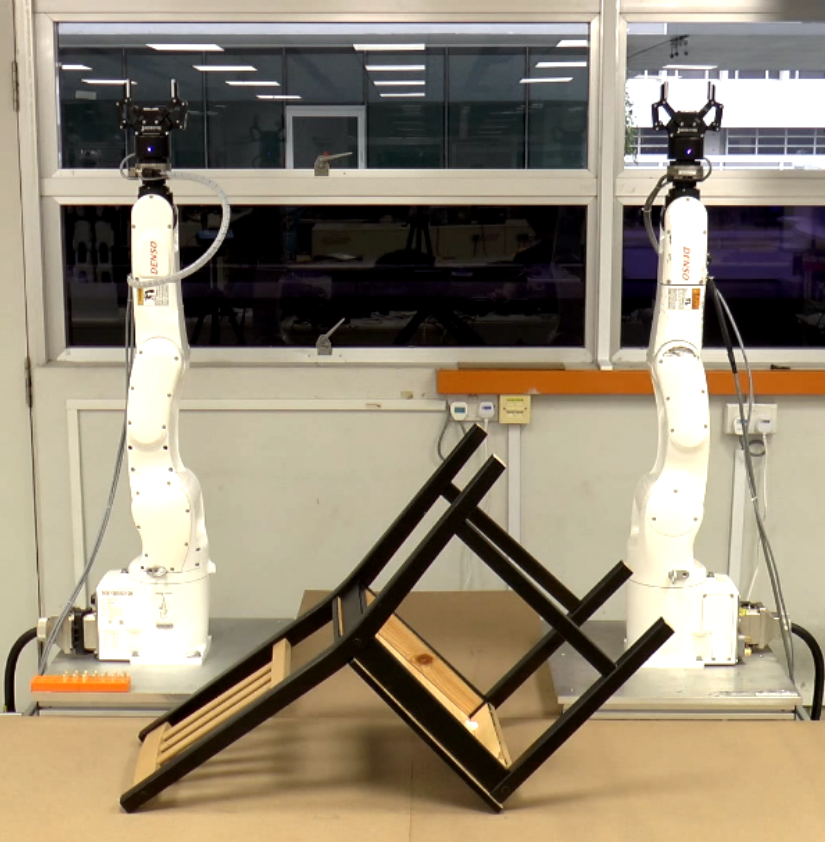}
    \label{figure:Tgoal}
  }
  \caption{Key object transformations along the computed bimanual
    manipulation trajectory solving the given query.}
  \label{figure:bquery}
\end{figure}

\section{Discussion and Conclusion}
\label{section:conclusion}

\subsection{Discussion}
The current algorithm has some limitations. First, since the method of
generating a certificate relies on a heuristic, despite its capability
illustrated in our experiments, there might be cases where the method
fails to produce a transfer path when there is one. To alleviate this
issue, when planning a \TB path, one might also allow intermediate
placements (which does not necessarily have to be a stable one) as
suggested in~\cite{XLP16arxiv} so as to allow regrasping such that the
robots can be at a configuration in the same single-transfer connected
component as the goal configuration. Second, when planning a \TA path,
we currently assume that all collision-free dragging or pulling
motions are feasible. This might not always be the case when the
contact between the object and the support surface has very high
friction and/or the robots have low maximum grip force. Future work
may include investigation of effects of these issues to \TA
connectivities.

\subsection{Conclusion}
In this paper, we first present a set of definitions and fundamentals
of bimanual manipulation planning. In order to solve a bimanual
manipulation query, it is essential for the planner to obtain
information of connectivities between different connected components
of the composite configuration space. We propose an algorithm which
constructs a manipulation solution by generating and concatenating two
types of trajectories: \TA trajectories which connect configurations
in the same placement class, and \TB trajectories which connect
configurations from different placement classes. The key specificity
of our algorithm is that it is certified-complete. We provide a method
to compute a \emph{certificate} for a given object and environment. A
certificate, once obtained, guarantees that the algorithm will find a
solution to any feasible bimanual manipulation query for the object in
that setting in finite time. Information contained in a certificate
can be used to construct a solution trajectory. Simulation and
experimental results illustrate the validity and capability of our
algorithm to plan bimanual manipulation motions for various practical
objects.

\section*{Acknowledgment}
This work was supported by Tier 1 grant RG109/14 awarded by the
Ministry of Education of Singapore.

\bibliographystyle{IEEEtran} 
\bibliography{../../CRI_jr}

\end{document}